\documentclass[final,12pt]{alt2024} 

\title[Adversarial Online Collaborative Filtering]{Adversarial Online Collaborative Filtering}
\usepackage{times}

\usepackage{microtype}
\usepackage{graphicx}
\usepackage{booktabs} 

\usepackage{hyperref}

\usepackage{amsmath}
\usepackage{amssymb}
\usepackage{mathtools}
\usepackage{amsmath}
\usepackage{algorithm, algorithmic}

\usepackage{wrapfig}

\usepackage{microtype}
\usepackage{xspace}
\usepackage{graphicx}
\usepackage{booktabs} 
\usepackage{xspace}

\usepackage[utf8]{inputenc} 
\usepackage[T1]{fontenc}    
\usepackage{hyperref}       
\usepackage{url}            
\usepackage{booktabs}       
\usepackage{amsfonts}       
\usepackage{nicefrac}       
\usepackage{xcolor}         

\usepackage[capitalize,noabbrev]{cleveref}

\usepackage{amsmath}
\usepackage{algorithm, algorithmic}

\usepackage{amssymb}
\usepackage{url}
\usepackage{enumerate}
\usepackage{sidecap}

\usepackage{stmaryrd}
\usepackage{textcomp}

\usepackage[utf8]{inputenc} 
\usepackage[T1]{fontenc}    
\usepackage{hyperref}       
\usepackage{amsfonts}       
\usepackage{nicefrac}       
\usepackage{microtype}      
\usepackage{soul} 

\usepackage{csquotes}

\newcommand{\knds}{\kern-\nulldelimiterspace}
\definecolor{verylightgray}{rgb}{0.8,0.8,0.8}
\definecolor{am}{rgb}{0.6, 0.29, 0.0}
\definecolor{ao}{rgb}{0.0, 0.5, 0.0}
\definecolor{blue-violet}{rgb}{0.54, 0.17, 0.89}
\definecolor{dartmouthgreen}{rgb}{0.05, 0.5, 0.06}
\definecolor{brown}{rgb}{0.6, 0.29, 0.0}

\newcommand\MJH[1]{\textcolor{blue}{{\bf MJH:}{\small \texttt{~#1~}}}}
\newcommand\FV[1]{\textcolor{brown}{{\bf FV:}{\small \texttt{~#1~}}}}
\newcommand\CG[1]{\textcolor{blue-violet}{{\bf CG:}{\small \texttt{~#1~}}}}

\newcommand\comm[1]{\textcolor{ao}{{\small \texttt{// #1}}}}

\newcommand{\bs}[1]{\boldsymbol{#1}}
\newcommand{\nc}[1]{\newcommand{#1}}
\nc{\nmu}{M}
\nc{\nma}{N}
\nc{\us}{[\nmu]}
\nc{\as}{[\nma]}
\nc{\emp}[1]{\emph{#1}}
\nc{\lik}{\sim}
\nc{\au}{i}
\nc{\ai}{j}
\nc{\ntr}{T}
\nc{\tim}{t}
\nc{\nat}{\mathbb{N}}
\nc{\ut}[1]{i_{#1}}
\nc{\itt}[1]{j_{#1}}
\nc{\timp}{s}
\nc{\na}[1]{[#1]}
\nc{\be}{\begin{equation*}}
\nc{\ee}{\end{equation*}}
\nc{\mis}{\mu}
\nc{\indi}[1]{\llbracket #1 \rrbracket}
\nc{\nlik}{\not\sim}
\nc{\re}{\lambda}
\nc{\ntu}[1]{\tau(#1)}
\nc{\tlik}[1]{\phi(#1)}
\nc{\reg}{R}
\nc{\eqi}{\equiv}
\nc{\uc}{C}
\nc{\ic}{D}
\nc{\algn}{\textsc{Orca}}
\nc{\bigo}[1]{\mathcal{O}(#1)}
\nc{\expt}[1]{\mathbb{E}[#1]}

\nc{\random}{\textsc{Random}}
\nc{\wrmf}{\textsc{Wrmf}}
\nc{\pop}{\textsc{Pop}}
\nc{\orcapop}{\textsc{OrcaPop}}

\nc{\lel}{\ell^*}
\nc{\lii}{\ell'}
\nc{\lij}{\ell''}
\nc{\lir}[1]{r_{#1}}
\nc{\lis}[1]{C_{#1}}
\nc{\lit}[1]{D_{#1}}
\nc{\lsg}[1]{S_{#1}}
\nc{\pos}[1]{\ell_{#1}}
\nc{\usd}[1]{\Phi(#1)}
\nc{\nequ}{\not\eqi}
\nc{\la}{\leftarrow}
\nc{\ben}{\begin{enumerate}}
\nc{\een}{\end{enumerate}}
\nc{\sig}{\Psi}
\nc{\cus}{Z}
\nc{\cpo}{\ell}
\nc{\uco}{\mathcal{R}}
\nc{\lat}{\Leftarrow}
\nc{\nn}{w}

\nc{\lil}{\approx}
\nc{\ucr}{\phi}
\nc{\icr}{\psi}
\nc{\ars}{W}
\nc{\ext}[1]{m}
\nc{\nul}[1]{c_{#1}}
\nc{\nnul}[2]{c_{#1,#2}}
\nc{\uca}{\textsc{UC}}
\nc{\ica}{\textsc{IC}}
\nc{\iic}[1]{B_{#1}}
\nc{\ius}[1]{A_{#1}}
\nc{\nic}[1]{Q_{#1}}
\nc{\nus}[1]{P_{#1}}
\nc{\pfin}{\Lambda}
\nc{\neqi}{\not\eqi}
\nc{\nis}[1]{\gamma_{#1}}
\nc{\nit}[1]{\delta_{#1}}
\nc{\sif}[1]{\psi_{#1}}
\nc{\sih}[2]{\sigma_{#1,#2}}
\nc{\ome}{\omega}
\nc{\ist}[1]{\mathcal{I}_{#1}}
\nc{\les}[1]{\mathcal{U}_{#1}}
\nc{\zi}[1]{\mathcal{P}_{#1}}
\nc{\lse}{k}
\nc{\crt}{\mathcal{K}}
\nc{\cru}[1]{u_{#1}}
\nc{\exs}{E_{\ucr}}
\nc{\nlil}{\not\lil}
\nc{\ssc}{H}
\nc{\lrel}[2]{L_{#1,#2}}
\nc{\lmat}{\bs{L}}
\nc{\lmap}{\bs{L}^*}
\nc{\lrep}[2]{L^*_{#1,#2}}
\nc{\hta}[1]{\widehat{\omega}_{#1}}
\nc{\hx}[1]{\widehat{\xi}_{#1}}

\nc{\byp}[1]{\mathcal{B}_{#1}}
\nc{\cnt}[1]{d_{#1}}
\nc{\exc}{\mathcal{E}}
\nc{\exi}{\mathcal{F}}
\nc{\coin}{\gamma}
\nc{\ue}{\textsc{UE}}
\nc{\uie}{\textsc{UIE}}
\nc{\uiebold}{{\bf \textsc{UIE}}}

\nc{\algnn}{\algn$^*$}
\nc{\ucp}[1]{\delta_{#1}}
\nc{\icp}[1]{\epsilon_{#1}}
\nc{\nqu}[1]{\omega_{#1}}
\nc{\nli}[1]{\xi_{#1}}
\nc{\bo}[1]{\mathcal{O}(#1)}

\nc{\gi}{\mathcal{G}}
\nc{\fnd}{\mathcal{H}}
\nc{\fnl}{\fnd^*}
\nc{\fnb}{\fnd^{\bullet}}
\nc{\fnc}{\fnd^{\circ}}
\nc{\clus}{\mathcal{K}}
\nc{\cls}{\clus'}
\nc{\cl}[1]{\clus_{#1}}
\nc{\rem}[1]{\rho_{#1}}
\nc{\fbig}{\fnd^{\dag}}
\nc{\fst}{\tau}
\nc{\fsl}{s}
\nc{\fc}{\mu}
\nc{\arbn}{x}
\nc{\fbd}{\mathcal{B}}
\nc{\ap}{\mathcal{A}'}

\altauthor{%
 \Name{Stephen Pasteris} \Email{spasteris@turing.ac.uk}\\
 \addr The Alan Turing Institute, UK
 \AND
 \Name{Fabio Vitale} \Email{fabio.vitale@centai.eu }\\
 \addr CENTAI, Italy%
 \AND
 \Name{Mark Herbster} \Email{m.herbster@ucl.ac.uk }\\
 \addr University College London, UK%
 \AND
 \Name{Claudio Gentile} \Email{cgentile@google.com}\\
 \addr Google Research, USA%
 \AND
 \Name{Andr\'e Panisson} \Email{andre.panisson@centai.eu}\\
 \addr CENTAI, Italy%
}

\begin{document}

\maketitle

\begin{abstract}%
We investigate the problem of online collaborative filtering under no-repetition constraints, whereby users need to be served content in an online fashion and a given user cannot be recommended the same content item more than once. We start by designing and analyzing an
algorithm that works under biclustering assumptions on the user-item preference matrix, and show that this algorithm exhibits an optimal regret guarantee, while being fully adaptive, in that it is oblivious to any prior knowledge about the sequence of users, the universe of items, as well as the biclustering parameters of the preference matrix.
We then propose a more robust version of this algorithm which operates with general matrices. Also this algorithm is parameter free, and we prove regret guarantees that scale with the amount by which the preference matrix deviates from a biclustered structure.
To our knowledge, these are the first results on online collaborative filtering that hold at this level of generality and adaptivity under no-repetition constraints.
Finally, we complement our theoretical findings with simple experiments on real-world datasets aimed at both validating the theory and empirically comparing to standard baselines. This comparison shows the competitive advantage of our approach over these baselines.%
\end{abstract}

\begin{keywords}%
  Online Learning, Collaborative Filtering, No-repetition constraint, Biclustering.%
\end{keywords}

\section{Introduction}\label{s:intro}

Helping customers identify their preferences is essential for businesses with a diverse product offering. Many companies rely on recommendation systems (RS)~\citep{resnick1997recommender}, which allow users to browse, search, or receive suggestions from online services. Recommendation algorithms let us narrow down massive amounts of information into personalized choices.
This is especially relevant in 
online businesses, where the capabilities of interactive RS have become of paramount importance.
Customers can obtain suggestions for movies (e.g., Netflix, YouTube TV), music (e.g., Spotify, YouTube Music), job openings (e.g., LinkedIn, Indeed), or various products (e.g., Amazon, eBay), while their feedback is tracked and exploited to improve future recommendations tailored to specific user interests. 
In most cases, the goal is to improve user experience as measured by the amount of ``likes'' 
given by the user over time. 

A standard approach to content recommendation is the one provided by Collaborative Filtering (CF), where personalized recommendations are generated based on both content data and aggregate user activity.
%
%
%
CF algorithms are either user-based or item-based. 
A user-based CF algorithm suggests items that similar users enjoy. Item-based algorithms suggests items that are similar to items enjoyed by the user in the past. In many practical applications, suggesting an item that has already been consumed is often useless~\citep{bresler2014latent, bresler2016collaborative, ariu2020regret}. For instance, it is pointless to keep advertising the same movie or book to a user after she has already watched or read it. In the recent online CF literature, this assumption is often called the \enquote{\em{no-repetition constraint}}~(e.g., \cite{ariu2020regret}). 

The aim of this paper is to design and analyze novel learning algorithms for online collaborative filtering 
under the no-repetition constraint assumption, still forcing the algorithms to leverage the collaborative effects in the user-item structure. 
In this sense, the algorithms we propose combine both user-based and item-based online CF approaches.

Our learning problem can be described as follows. Learning proceeds in a sequence of interactive {\em trials} (or {\em rounds}). At each round, a single user shows up, and the RS is compelled to recommend content (an individual item) to them. We make no assumptions whatsoever on the way the sequence of users gets generated across rounds. 
%
%
The user then provides feedback encoding their opinion about the selected item, and the RS uses this signal to update its internal state. 
The kind of feedback we expect is the binary click/no-click, thumb up/down, like/dislike, which is very common in online media services (TikTok, YouTube, etc.)
%
%
In any trial, the RS is constrained to recommend to the user at hand an item that {\em has not} been recommended to that user in the past.

In order to leverage non-trivial collaborative effects, we
investigate a latent model of user preferences based on {\em biclustering}~\citep{H72}, or perturbed/noisy versions thereof. 
Specifically, we shall assume that each user falls under one user {\em type} (or {\em cluster}) and that, in the absence of noise, users within the same type prefer the same items. At the same time,
items are also clustered so that, in the absence of noise, all items belonging to the same cluster are liked or disliked in the same fashion by each user.  Modern user-based (item-based) CF algorithms often operate under these assumptions, which are supported by a number of experiments in the RS literature, for both user~\citep{das2007google, bellogin2012using, bresler2014latent, bresler2021regret} and item~\citep{sarwar2001item, bresler2016collaborative, bresler2021regret} clustering.
The set of all user preferences is generally viewed as a matrix called a {\em preference matrix}. A perturbed version of a biclustered matrix is one where the actual preference matrix can be seen as a noisy variant of a biclustered matrix, a scenario we also tackle in our analysis.

\noindent{\bf Our contributions.}
%
We consider a sequential and adversarial learning setting where users and user preferences can be generated {\em arbitrarily}. We initially investigate the case where the preference matrix is perfectly biclustered (Section \ref{s:noisyfree}), and then relax this assumption to allow an adversarial perturbation of the preference matrix (Section \ref{s:noisy}). In all cases, we quantify online performance in terms of cumulative {\em regret}, that is, the extent to which the number of recommendation mistakes made by our algorithms across a sequence of rounds exceeds those made by an omniscient oracle that knows the preference matrix beforehand. In the perfect biclustering case, we fully characterize the problem: we both provide a regret lower bound and describe an algorithm that achieves a regret guarantee that matches this lower bound. In the adversarially perturbed case, we introduce a more robust algorithm that operates with general preference matrices, and whose regret performance is expressed in terms of the degree by which the perturbed preference matrix diverges from a perfectly biclustered one.
Our algorithms are scalable, {\em fully adaptive} (that is, {\em parameter free}), in that they need not know the parameters of the underlying biclustering structure, the time horizon, or the amount of perturbation in the preference matrix. Moreover, the algorithms can be naturally run in situations where both the set of users and the universe of items increase (arbitrarily) over time. 

We then empirically compare (Section \ref{s:experiments}) versions of our algorithm to three baselines (recommendation based on popularity, the Weighted Regularized Matrix Factorization (WRMF) from \cite{koren:icdm08}, and random recommendations) on a real-world benchmark showing that our algorithms exhibit on this benchmark faster learning curves than its competitors.

\noindent{\bf Discussion and related work.}
As far as we are aware, this is the first work that provides theoretical performance guarantees under the no-repetition constraint in a non-stochastic setting where both the sequence of users and the user-item preference matrix are generated adversarially.


Because user preferences can be observed solely for items recommended in the past, we must address the classical problem of how to quickly determine the
users' interests without affecting the quality of our recommendations. 
That is, 
should we provide recommendations according to the users' interests observed thus far, 
or should we obtain new feedback signals so as to better profile the users? This well-known exploitation-exploration dilemma 
has received wide attention in the machine learning and statistics fields. In particular, the research in multi-armed bandit (MAB) problems and methods applied to RS tasks
is interested in how to strike an optimal balance
between exploitation of current knowledge of user preferences and exploration of new potential interests~(see, e.g., the monographs by \cite{bubeck2012regret} and \cite{lattimore2020bandit}).


One thing to emphasize, though, is that
none of the variants of MABs which are readily available in the literature includes all the core elements of the problem we are considering here. An item can only be suggested {\em once} to a user in our setting, while an arm (item) in MAB problems can typically be recommended multiple times.
This is a significant difference between our setting and standard MAB formulations where, once the best arm is discovered for a given user, the problem for that user is deemed to be solved.
%
Clustering~(e.g., \cite{bui2012clustered,maillard2014latent, gentile2014online,  10.1145/2911451.2911548, kwon2017sparse, jedor2019categorized}) as well as low-rank (e.g., \cite{katariya+17,10.1145/3240323.3240408,NEURIPS2020_9b7c8d13,pmlr-v97-jun19a,pmlr-v117-trinh20a,pmlr-v130-lu21a,kang2022efficient}) assumptions are also widespread in the stochastic MAB literature as applied to recommendation problems, but these works do not consider the no-repetition constraint on the items, nor do they address adversarial perturbations of the preference matrix.




The reader is referred to Appendix \ref{sa:related} for further discussion on the connections to matrix completion/factorization and structured bandit formulations.

The closest references to our work are perhaps the works by \cite{bresler2014latent, bresler2016collaborative, ariu2020regret,bresler2021regret}, where online CF problems are investigated under the no-repetition constraint assumption. As in our paper, algorithm performance is measured
by comparing the proposed algorithm against an omniscient RS that knows the preferences of all users on all items. However, \cite{bresler2016collaborative} assumes the sequence of users is generated in a quite benign way (uniformly at random), while in the papers by \cite{bresler2014latent, ariu2020regret,bresler2021regret} all users need to receive a recommendation and provide a feedback simultaneously.

Finally, it is worth mentioning that standard ranking problems, where the RS is required to produce a ranked list of diverse items (see, e.g., classical references, like \cite{10.5555/2567709.2502595}), can be seen as a way to implement the no-repetition constraint, but only within a given user session, not across multiple sessions of the same user. Hence this gives rise a substantially different RS problem than the one we consider here.

\section{Preliminaries, Learning Tasks, and Overview of Results}\label{s:basic}
%
All the tasks considered here involve the recommendation of $\nma$ items to $\nmu$ users. In order to define our problems we must first introduce what it means for a matrix to be \emp{(bi)clustered}.

Given an $\nmu\times \nma$ matrix $\lmat$, we say that two users $\au,\au'\in\na{\nmu}$ are \emp{equivalent} if and only if $L_{\au,\ai}=L_{\au',\ai}$ for all $\ai\in\na{N}$, that is, if and only if the rows corresponding to the two users are identical. Similarly, two items $\ai,\ai'\in\na{\nma}$ are \emp{equivalent} if and only if $L_{\au,\ai}=L_{\au,\ai'}$ for all $\au\in\na{\nmu}$. A matrix $\bs{L}$ is $C$-user clustered and $D$-item clustered if the number of equivalence classes under these equivalence relations are no more than $C$ and $D$, respectively. For brevity, we shall refer to such a matrix as a $(C,D)$-biclustered matrix.

\noindent{\bf The Basic Problem. }
We first describe the simplest version of the problem that we study. We have an unknown binary matrix $\bs{L}$ which is $(C,D)$-biclustered for some unknown $C$ and $D$. We say that user $i$ \emp{likes} item $j$ if and only if $L_{i,j}=1$. Our learning problem proceeds sequentially in trials (or rounds)  $t=1,2,\ldots,T$, where on trial $t$ a learning agent (henceforth called ``Learner")
interacts with its environment as follows:
%
\begin{enumerate}
\item The environment reveals user $i_t$ to Learner;
\item Learner chooses an item $j_t$ to recommend to $i_t$. However, Learner is restricted in that it cannot have recommended item $j_t$ to user $i_t$ on some earlier trial;
\item $L_{\ut{\tim},\itt{\tim}}$ is revealed to Learner.
\end{enumerate}
%
Note that the problem restricts the environment in that a given user cannot be queried more than $\nma$ times. For any trial $t$, if $L_{\ut{\tim},\itt{\tim}}=0$ then user $\ut{\tim}$ does not like item $\itt{\tim}$ and we say that Learner incurs a \emp{mistake}. The aim of Learner is to minimize the total number of mistakes made throughout the $T$ rounds for the given matrix $\bs{L}$ and the sequence of users $i_1,\ldots, i_T$ generated by the environment.

In fact, since the binary matrix $\bs{L}$ can be arbitrary (it may contain a lot of zeros) and the sequence $i_1,\ldots, i_T$ may be generated adversarially, our goal will be to bound the learner's \emp{regret} $R$, which is defined as the difference between the number of mistakes made by Learner and those which would have been obtained by an {\em omniscient} oracle that has a-priori knowledge of $\bs{L}$. Formally, given a user $\au\in\na{\nmu}$, let $\nqu{\au}$ be the number of rounds $t\in\na{\ntr}$ in which $\au=\ut{\tim}$, and let $\nli{\au}$ be the number of items in $\na{\nma}$ that $\au$ likes.
Let us denote here and throughout by the brackets $\indi{\cdot}$ the indicator function of the predicate at argument. The regret $\reg$ is then defined as:
\be
\reg:=\sum_{\tim\in\na{\ntr}}(1-\lrel{\ut{\tim}}{\itt{\tim}})-\sum_{\au\in\na{\nmu}} (\omega_i - \xi_i)\indi{\omega_i \geq \xi_i}~,
\ee
where the first summation is the number of mistakes made by Learner and the second summation is the number of mistakes made by the omniscient oracle for the given sequence of users $i_1,\ldots, i_T$. 
We shall prove for this problem that our 
randomized algorithm \algn\ ({\bf O}ne-time {\bf R}e{\bf C}ommendation {\bf A}lgorithm  -- Section \ref{ss:orcbandit}) has an expected regret bound of the form:
\be
\expt{R} = \mathcal{O}\left(\min\{C,D\}(\nmu+\nma)\right)~,
\ee
and a time complexity of only $\mathcal{O}(\nma)$ per round. We will also prove that the above regret bound is essentially optimal. The algorithm has to be randomized since it is designed to deal with adversarially generated user sequences (the same applies to our second algorithm \algnn).

\noindent{\bf 
General preference matrices. }
We now turn to the problem of incorporating adversarial perturbation into matrix $\lmat$. In this case we will only exploit similarities among items and leave it as an open problem to adapt our methodology to exploit similarities among users. Our algorithm \algnn\ (Section \ref{ss:orcbanditstar}) takes an integer parameter $\icr\geq 2$. 
We shall assume now that we have a \emp{hidden} $D$-item clustered $\nmu\times\nma$ binary matrix $\lmap$ which is perturbed {\em arbitrarily} to form the matrix $\lmat$. The matrix $\lmap$ and parameter $\icr$ induce the following concept of {\em bad} users and {\em bad} items. Recall that $\nqu{i}$ and $\nli{i}$ are the number of times this user is queried and the number of items that it likes, respectively.
\begin{definition}\label{def:1}
Given a user $i\in[M]$, its \emp{perturbation level} $\ucp{i}$ is the number of items $j\in[N]$ in which $\lrel{i}{j}\neq\lrep{i}{j}$.  User $i$ is a \emp{bad} user if and only if both $\ucp{i}>0$ and $\nqu{i}>\nli{i}-2\ucp{i}$ hold.
Given an item $j\in[N]$, its \emp{perturbation level} $\icp{j}$ is the number of users $i\in[M]$ in which $\lrel{i}{j}\neq\lrep{i}{j}$. Item $j$ is a \emp{bad} item if and only if $\icp{j}>\icr$, for the given value of $\icr$.
\end{definition}
Note that the parameter $\icr$ does affect the definition of a bad item. 
In particular, when $\icr$ increases, the number of bad items decreases, and it does so in a way that depends on the structure of the (unknown) preference matrix. We can view $\icr$ as a tolerance of the algorithm to item perturbation. A bad user is one which the learner is compelled to serve content ``too often". Observe that this notion not only depends on the difference between $\bs{L}$ and $\bs{L^*}$, but also on the specific sequence of users $i_1,\ldots,i_T$ generated by the environment. We also stress that in relevant real-world scenarios there will often be no bad users. E.g., there are usually many more books that a person would enjoy than those they have time to read.

Given that we have $m$ bad users and
$n = n(\psi)$ bad items, \algnn\ enjoys the following regret bound:
\begin{align*}
\expt{\reg} = {\mathcal O}\Bigl(\min\Bigl\{ (D\icr+m+n)(M+N), (D+n+m/\icr)(M+N\icr)\Bigl\}\Bigl)\,.
\end{align*}
Note that the two terms in the above minimum are generally incomparable, even when solely viewed as a function of parameter $\psi$. When $M=\bo{N}$ the first term is better due to the reduced influence of $n$, whilst when $M=\Omega(N\icr)$ the second term is better. Hence the above bound expresses a best-of-both-worlds guarantee which is independent of the relative size of $M$ and $N$.
It is also instructive to consider how the two terms change as a function of $\icr$. As we said, when $\icr$ increases, $n$ decreases, and vice versa. However, since both terms in the above minimum also exhibit a linear dependence on $\icr$, there is typically a ``sweet spot" for $\icr$, which can easily be found, again in a fully adaptive way, as explained in Section \ref{ss:analysis}.



\noindent{\bf Dynamic Inventory. }
%
A natural extension to our problem is to dynamically allow new users and items over time.  Thus on a given trial we may or may not see a (single) new user, but the set of items $\mathcal{I}_t$ that we may recommend from on round $t$ is a superset of the items from previous round, i.e., $\mathcal{I}_{t-1}\subseteq \mathcal{I}_t$. Besides, there is no limit on the number of added items.  Conventionally, at the final round $T$ the set of distinct users is $[M]$ and distinct items is $[N]= \mathcal{I}_T$.  Our algorithms do not need to know $M$ and $N$ in advance.

For simplicity of presentation, we will only consider here the 
the perturbation-free case, 
but the same methodology can also be applied in the adversarially perturbed case. 
The notion of regret generalizes to this dynamic case as follows. 

For all trials $t = 1,\ldots, T$, let $\hta{t}$ be defined recursively as
$\hta{t} = 1 + \sum_{s < t} \indi{ i_s = i_t} \indi{\hta{s}\leq\hx{s}}$,
where $\hx{t}$ is the number of items in $\mathcal{I}_{t}$ that user $\ut{t}$ likes. The regret is then defined as
\be
\reg:=\sum_{\tim\in\na{\ntr}}(1-\lrel{\ut{\tim}}{\itt{\tim}})-\sum_{\tim\in\na{\ntr}}\indi{\hta{t}>\hx{t}}\,.
\ee
We note that with the above definition $\hta{t}$, the regret $R$ is again the difference between the number of mistakes made by Learner and those of an omniscient oracle. 
For this dynamic case, \algn\ enjoys {\em the same} regret guarantee as it did for the static case. However, its running time increases from $\mathcal{O}(\nma)$ to
$\mathcal{O}(\nma^2)$ per round. (More precisely, $\mathcal{O}(|\mathcal{I}_{T}|^2)$ per round, where $|\mathcal{I}|$ is the cardinality of set $\mathcal{I}$.)

\section{The Perturbation-Free Case}\label{s:noisyfree}
%
We start off by describing the basic algorithm
\algn\ which is designed to work in the absence of perturbation (that is, in the purely biclustered case) when the inventory is either static 
or dynamic. 
In the dynamic case the time complexity of the algorithm is $\mathcal{O}(|\mathcal{I}_{T}|^2) = \mathcal{O}(\nma^2)$ per trial. On the other hand, in the static case (i.e., when $\ist{\tim}=\na{\nma}$ for all $\tim\in\na{\ntr}$) the time complexity decreases to $\mathcal{O}(\nma)$ per trial. We shall analyze both the static and dynamic cases together.

\begin{algorithm}[!h] 
{\bf Initialization :}
$\lel\la0;$~~~
$\pos{\au}\la0$ for all $\au\in\nat$;\\
{\bf For} $t = 1, \ldots, T$ :
\begin{enumerate}
\item
$\cpo\la\pos{\ut{\tim}}$;
\item $\uco\leftarrow$ set of all items in $\ist{\tim}$ not recommended yet to $\ut{\tim}$;

\item \label{tp1} {\bf If} there exists an item $j_t\in\uco$ and a level $k\neq0$ with $k\leq\cpo$\,, $i_t\in\les{k}$ and $j_t\in \zi{k}$ 
{\bf then :}\\ 
%
\vspace{-0.4cm}
\begin{itemize}
    \item Select $j_t$;~~~~~{\bf If} $\lrel{\ut{\tim}}{\itt{\tim}}=0$ {\bf then} remove $\itt{\tim}$ from the item pool $\zi{k}$;
\end{itemize}
%
\item \label{tp2} \label{trt2} {\bf Else if} $\cpo\neq\lel$
{\bf then :}
\begin{itemize}
\item {\bf If} $\lir{\cpo+1}\in\uco$ select item $\itt{\tim}\la\lir{\cpo+1}$~{\bf else} select any item $j_t$ from $\uco$;
\item 
$\pos{\ut{\tim}}\la\cpo+1$;
\end{itemize}
\item \label{tp3} {\bf Else :} 
%
\begin{itemize}
\item Select item $\itt{\tim}$ uniformly at random from $\uco$;
\item {\bf If} $\lrel{\ut{\tim}}{\itt{\tim}}=1$ {\bf then}: \hfill\comm{Create new level} 
\begin{itemize}
\item 
$\lel\la\lel+1$;~~~~~$\pos{\ut{\tim}}\la\lel$;~~~~~$\lir{\lel}\la\itt{\tim}$;~~~~~$\zi{\lel}\la\nat$;
\hfill\comm{Define $\cru{\lel}:=i_t$}
\item Define $\les{\lel}$ to be the set of all users $\au\in\us$ with:\\
$\Rsh$ Only for {\bf \uca:}\,\,
$\lrel{\au}{\lir{\lii}}=\lrel{\ut{\tim}}{\lir{\lii}}~~\forall \ell' \leq \lel$\\
\raisebox{\depth}{\scalebox{1}[-1]{\(\Rsh\)}}~Only for {\bf \ica:}\,\,\,\,\,
$\lrel{\au}{\lir{\lel}}=1$
\end{itemize}
\end{itemize}
\end{enumerate}
\caption{One time recommendation algorithm (\algn).\label{a:orca}}
\end{algorithm}
\subsection{The One-Time Recommendation Algorithm \algn}\label{ss:orcbandit}
%
\algn's pseudocode is contained in Algorithm \ref{a:orca}, and its online functioning is briefly illustrated in Figure~\ref{f:1}.

\algn\, has two variants -- \algn-UC and \algn-IC, which exploit user and item clusters, respectively. The two variants share much of the same code. We will refer for brevity to them as \uca\ and \ica.
These algorithms will be designed in such a way that they can be fused together (into the resulting \algn) as follows. We run both \uca\ and \ica\ in parallel and maintain a $0/1$-flag.
On any round $t$, if the flag is set to $0$ then we select item $\itt{\tim}$ with \uca\ and update \uca. If the flag is set to $1$ we do so with \ica. The flag gets flipped if and only if $\lrel{\ut{\tim}}{\itt{\tim}}=0$. 
The two algorithms run in parallel and do not share any information, except for the items already recommended to each user.
 
\uca\ and \ica\ also share much of the same analysis. Hence, when we describe and analyze the two algorithms we are considering both at the same time, unless we state otherwise. The pseudo-code of the two algorithms only differ in the lines marked ``\uca" and ``\ica" at the very end of Algorithm \ref{a:orca}.

The algorithm(s) maintains over time a partitioning of the previously observed users into a sequence of user sets, that we call {\em levels}. Each observed user is initially assigned to the level $0$, and at any trial, 
can only move forward to the next level or stay in its current one. Informally, the higher is the level of a user, the more we know about her item preferences. Hence, with this method we can profile users based on their observed preferences. The current total number of levels is denoted by $\lel$, and can only increase over time, when a user belonging to level $\lel$ moves forward to a new level, thereby increasing by $1$ the value of $\lel$.
For all levels $\lii$, we denote by $\cru{\lii}$ the user $\ut{\tim}$ on the trial $\tim$ on which level $\lii$ gets created - in that $\lel$ becomes equal to $\lii$ (in Line \ref{tp3} of the pseudo-code). Each level $\lii$ is associated with a \emp{representative} item $\lir{\lii}$ and a set of users $\les{\lii}$. The only difference between \uca\ and \ica\ is how this set $\les{\lii}$ is defined:\footnote
{
Note that at initialization, we set $\pos{\au}\la0$ for all $i$, yet, level $0$ is not a real level, and the sets $\les{0}$ and $\zi{0}$ are dummy sets that are only meant to simplify the pseudo-code.
}
\begin{itemize}
\item In \uca, $\les{\lii}$ is the set of all users $\au$ in which $\lrel{\au}{\lir{\lij}}=\lrel{\cru{\lii}}{\lir{\lij}}$ for all $\lij\leq \lii$.
\item In \ica, $\les{\lii}$ is the set of all users $\au$ with $\lrel{\au}{\lir{\lii}}=1$.
\end{itemize}

We note that, for all levels $\ell'$\,, the set $\les{\ell'}$ need not (and cannot) be constructed explicitly. When a user $\au$ moves into a level $\lii$ -- in that $\pos{i}$ becomes equal to $\lii$ (during step \ref{tp2} of the pseudo-code in Algorithm \ref{a:orca}) -- we check to see if $i$ likes $\lir{\lii}$. This means that at any point in time we know, for all $\lij\leq\pos{i}$, whether or not $i\in\les{\lij}$.

For each level $\lii$ we maintain a \emp{pool} $\zi{\lii}$ which is the set of all items that we believe are liked by all users in $\les{\lii}$. In essence, the algorithm is biased towards the initial assumption that all users like all items, and then operates by keeping track of any bias violation to recover the biclustering structure. This is because we initialize $\zi{\lii}$ to be equal to $\nat$ (the universe of potential items). When we encounter a user $\au$ that we know to be in $\les{\lii}$ and that we know does not like an item $\ai$, we then know that $\ai$ is not liked by all users in $\les{\lii}$, and hence we remove it from $\zi{\lii}$.\footnote
{
To keep the algorithm simple and fast we actually do not always remove this item, but for the purposes of this discussion assume we do.
}

On any trial $t$ we call an item $j\in\ist{\tim}$ {\em recommendable} if it has not been recommended yet to user $i_t$ and there exists a level $\lii\neq 0$ not larger than\footnote{In the static inventory model the condition on $\lii$ in this definition is that it is {\em equal} to current level of $i_t$.} the current level of $i_t$, such that $i_t\in\les{\lii}$ and $j_t\in \zi{\lii}$.

\begin{figure}[ht]
\vspace{-1.3in}\hspace{-0.1in}
\begin{minipage}{0.5\textwidth}
\hspace{-0.8in}
            \begin{flushright}
			\includegraphics[width=7.8cm]{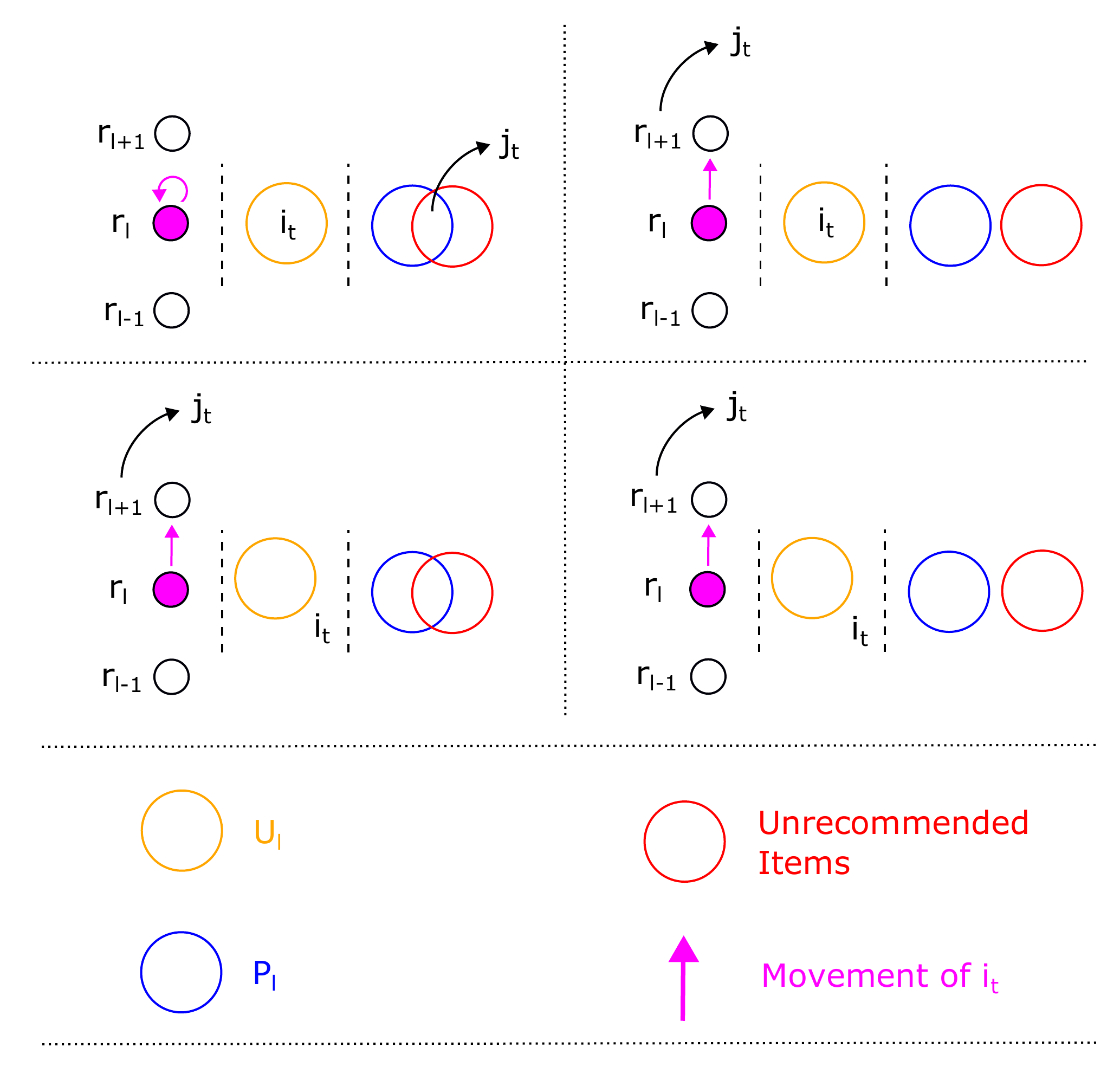}          
            \end{flushright}
\end{minipage}
\hspace{0.5cm}
\begin{minipage}{0.6\textwidth}
\vspace{1.0in}\hspace{-0.3in}
            \begin{flushleft}
			\includegraphics[width=7.2cm]{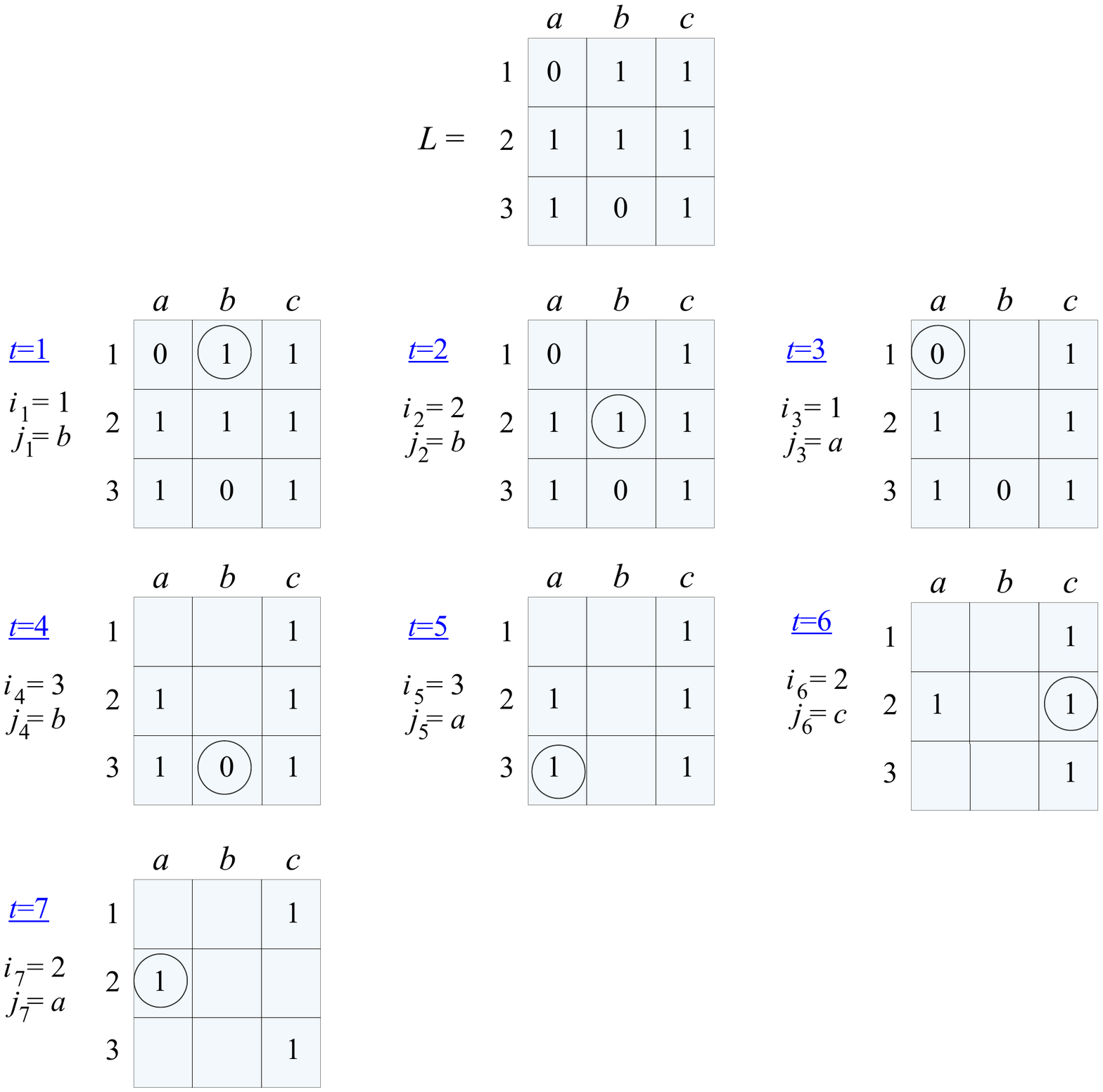}
			\end{flushleft}
\end{minipage}
\vspace{-1.0in}
\caption{{\bf Left:} The behavior of \algn\ (with a static inventory) at trial $t$ when $\ell:=\pos{\ut{\tim}}<\lel$. The four boxed diagrams on top represent the possible cases. In all cases, the yellow circle represents the set $\les{\ell}$, the red circle contains the set of so far unrecommended items for user $i_t$, and the blue circle represents set $\zi{\ell}$. The figure illustrates the movement of $i_t$ across levels. In the upper-left diagram the condition in Line \ref{tp1} is true. In all other diagrams the condition in Line \ref{tp2} is true. Note that in the upper-left diagram $j_t$ is removed from the blue set $\zi{\ell}$. The item $j_t$ is always removed from the red set of items not yet recommended to user $i_t$. When $\ell=\lel$ the upper-left diagram still applies.
{\bf Right:} An example run of 
\algn-UC with user set $\{1,2,3\}$, item set $\{a,b,c\}$, and a static inventory. 
\label{f:1}
}
%
\vspace{-0.28in}
\end{figure}

On any trial $t$ in which $\ell:=\pos{i_t}>0$\,, we recommend our item $\itt{\tim}$ and update the level of $i_t$ as follows. If there exists a recommendable item (i.e. the condition in Line \ref{tp1} is true) then we choose $\itt{\tim}$ to be such an item. If no such item exists (i.e. the condition in Line \ref{tp1} is false) then, given $\ell<\lel$ (so that the condition in Line \ref{tp2} is true), we move $\ut{\tim}$ up a level. We do so by first checking to see if $i_t$ likes the next level's representative item $\lir{\ell+1}$ (by choosing $j_t:=\lir{\ell+1}$ if it hasn't already been recommended this item) and then updating $\pos{\ut{\tim}}\la\ell+1$. Now suppose we want to move $i_t$ up a level but $\pos{\ut{\tim}}=\lel$ so there is no level to move up to (that is, the condition in Line \ref{tp3} is true). In this case, we draw $\itt{\tim}$ uniformly at random from those items in $\ist{\tim}$ not yet recommended to $\ut{\tim}$. If $\ut{\tim}$ likes $\itt{\tim}$ then we increment $\lel$ by one and then create a new level $\lel$ with $\lir{\lel}:=\itt{\tim}$ and $\cru{\lel}:=\ut{\tim}$.
\newline
\newline
\noindent{\bf Example run.}
For further clarity, in Figure \ref{f:1} (right) we give an example run of 
\algn-UC with user set $\{1,2,3\}$, item set $\{a,b,c\}$ and a static inventory. We now detail what happens on each trial:
\begin{itemize}
\vspace{-0.04in}
\item $t=1, i_1=1$. Since $\ell_1=0$ and $\ell^*=0$ the algorithm attempts to create a new level (level $1$) by sampling $j_1$ uniformly at random from $\mathcal{R}=\{a,b,c\}$. We draw $b$ so set $r_1=b$. Since $L_{1,b}=1$ user $1$ moves into level $1$ so now $\ell_1=1$ and $\ell^*=1$. We initialise $\mathcal{P}_1=\{a,b,c\}$. We note that if it had been the case that $L_{1,b}=0$ then user $1$ would have stayed at level $0$ and no new level would have been created.
\vspace{-0.07in}
\item $t=2, i_2=2$. Since $\ell_2=0$ and $\ell^*=1$, user $2$ is move into level $1$, and hence $j_2=r_1=b$, and now $\ell_2=1$.
\vspace{-0.07in}
\item $t=3, i_3=1$. We have $\ell_1=1$. Also $\mathcal{R}=\{a,c\}$ so $\mathcal{R}\cap\mathcal{P}_1\neq\emptyset$ and hence the algorithm chooses any $j_3$ from $\mathcal{R}\cap\mathcal{P}_1=\{a,c\}$. Say, we choose $j_3=a$. Since $L_{1,a}=0$ we remove $a$ from $\mathcal{P}_1$.
\vspace{-0.07in}
\item $t=4, i_4=3$. Since $\ell_3=0$ and $\ell^*=1$, user $3$ is move into level $1$, hence $j_4=r_1=b$, and now $\ell_3=1$.
\vspace{-0.07in}
\item $t=5, i_5=3$. Since $L_{3,r_1}=0$ we have that $3\notin\mathcal{U}_1$ and hence, since $\ell_3=1$ and $\ell^*=1$, the algorithm attempts to create a new level by sampling $j_5$ uniformly at random from $\mathcal{R}=\{a,c\}$. Say we draw $j_5=a$. Since $L_{3,a}=1$, a new level is created. Now $\ell^*=\ell_3=2$ and $r_2=j_5=a$. We also initialise $\mathcal{P}_2=\{a,b,c\}$. We note that if it had been the case that $L_{3,a}=0$ then user $3$ would have stayed at level $1$ and no new level would have been created.
\vspace{-0.07in}
\item $t=6, i_6=2$. We have $\ell_2=1$, $\mathcal{R}=\{a,c\}$ and $\mathcal{P}_1=\{b,c\}$. So $\mathcal{R}\cap\mathcal{P}_1=\{c\}$ and hence $j_6=c$.
\vspace{-0.07in}
\item $t=7, i_7=2$. We have $\ell_2=1$, $\mathcal{R}=\{a\}$ and $\mathcal{P}_1=\{b,c\}$. So $\mathcal{R}\cap\mathcal{P}_1=\emptyset$ and hence, since $\ell^*>\ell_2$, user $2$ is moved up to level $2$, so that $\ell_2=2$. Since $r_2=a$ and hence $r_2\in\mathcal{R}$ we choose $j_7=r_2=a$. We note that if it was the case that $r_2\notin\mathcal{R}$ then $j_t$ would have been chosen from $\mathcal{R}$.
\end{itemize}

\vspace{-0.05in}
\subsection{Analysis}
\vspace{-0.05in}
We analyze the properties of \algn\ by giving an upper bound on its regret as well as a matching lower bound on the regret of {\em any} algorithm (hence showing the optimality of \algn).
\begin{theorem}\label{thm:noisefree}
Let \algn\ be run on a $(C,D)$-biclustered matrix $\bs{L}$ of size $M\times N$ with an arbitrary sequence of users $i_1,\ldots, i_T$ and an arbitrary monotonically increasing sequence of item sets $\ist{1},\ldots, \ist{T}\subseteq[N]$. Then the expected regret of \algn\ is upper bounded as
\be
\expt{\reg} = \mathcal{O}(\min\{C,D\}(\nmu+\nma))~,
\ee
the expectation being over the internal randomization of the algorithm. \algn\ is parameter-free in that $C,D,M,N$ need not be known.
\end{theorem}

\begin{proof} Here we sketch the proof, deferring the full proof to Appendix \ref{pth2}.
Without loss of generality we assume that on every trial $\tim$ there exists an item $\ai\in\ist{\tim}$ which $\ut{\tim}$ likes and has not been recommended to $\ut{\tim}$ before so that the regret equals the number of mistakes. Let $\pfin$ be the total number of levels created by Algorithm \ref{a:orca}.

Given a level $\ell'\in[\pfin]$ we count the following mistakes. There are at most $N$ mistakes made on trials corresponding to Line \ref{tp1} of Algorithm \ref{a:orca} with $k=\ell'$, at most $M$ mistakes on trials corresponding to Line \ref{tp2} with $\cpo=\ell'-1$ and, in expectation, at most $N$ mistakes made on trials corresponding to Line \ref{tp3} with $\cpo=\ell'-1$. This means that the total number of mistakes is $\mathcal{O}(\pfin(M+N))$. 

Next, we proceed to bounding $\pfin$. A crucial property here is what we call the \emp{separation property}, stating that given $\lii,\lij\in\na{\pfin}$ with $\lij>\lii$ and $\cru{\lij}\in\les{\lii}$ there exists some $\au\in\les{\lii}$ with $\lrel{\au}{\lir{\lij}}=0$. With this property in hand, we then analyze the algorithms \uca\ and \ica\ separately. 

In \ica\ the separation property implies that given distinct levels $\lii,\lij\in[\pfin]$, their representative items $\lir{\lii}$ and $\lir{\lij}$ are not equivalent. This directly implies that $\pfin\leq\ic$.

In \uca\ the separation property leads to what we call the \emp{tree property}, stating that given levels $\lii,\lij\in[\pfin]$ with $\lii<\lij$ we have either $\les{\lij}\cap\les{\lii}=\emptyset$ or $\les{\lij}\subset\les{\lii}$. We build a forest (whose nodes are sets) as follows. First, for every level $\lij\in[\pfin]$ we make $\les{\lij}$ a node. If there exists $\lii<\lij$ with $\les{\lij}\subset\les{\lii}$ then for the maximum such $\lii$ we make $\les{\lii}$ the parent of $\les{\lij}$. Otherwise $\les{\lij}$ has no parent. Finally, if there exists $\lii\in[\pfin]$ such that $\les{\lii}$ has a single child $\les{\lij}$ then we add the set $\les{\lii}\setminus\les{\lij}$ to the forest and make it a child of $\les{\lii}$. Figure 2 gives an example of this forest. By the tree property all leaves are disjoint and non-empty, which implies that there are at most $\uc$ leaves. Since each internal node of the forest has at least two children, and there are at least $\pfin$ nodes, this yields $\pfin\leq 2\uc$.

We have hence shown that the mistakes of \ica\ and \uca\ are $\mathcal{O}(\ic(M+N))$ and $\mathcal{O}(\uc(M+N))$, respectively. When combined together into \algn\ we hence get the desired bound.
\end{proof}

\begin{figure}[h]
\centering
\includegraphics[width=0.65\textwidth]{{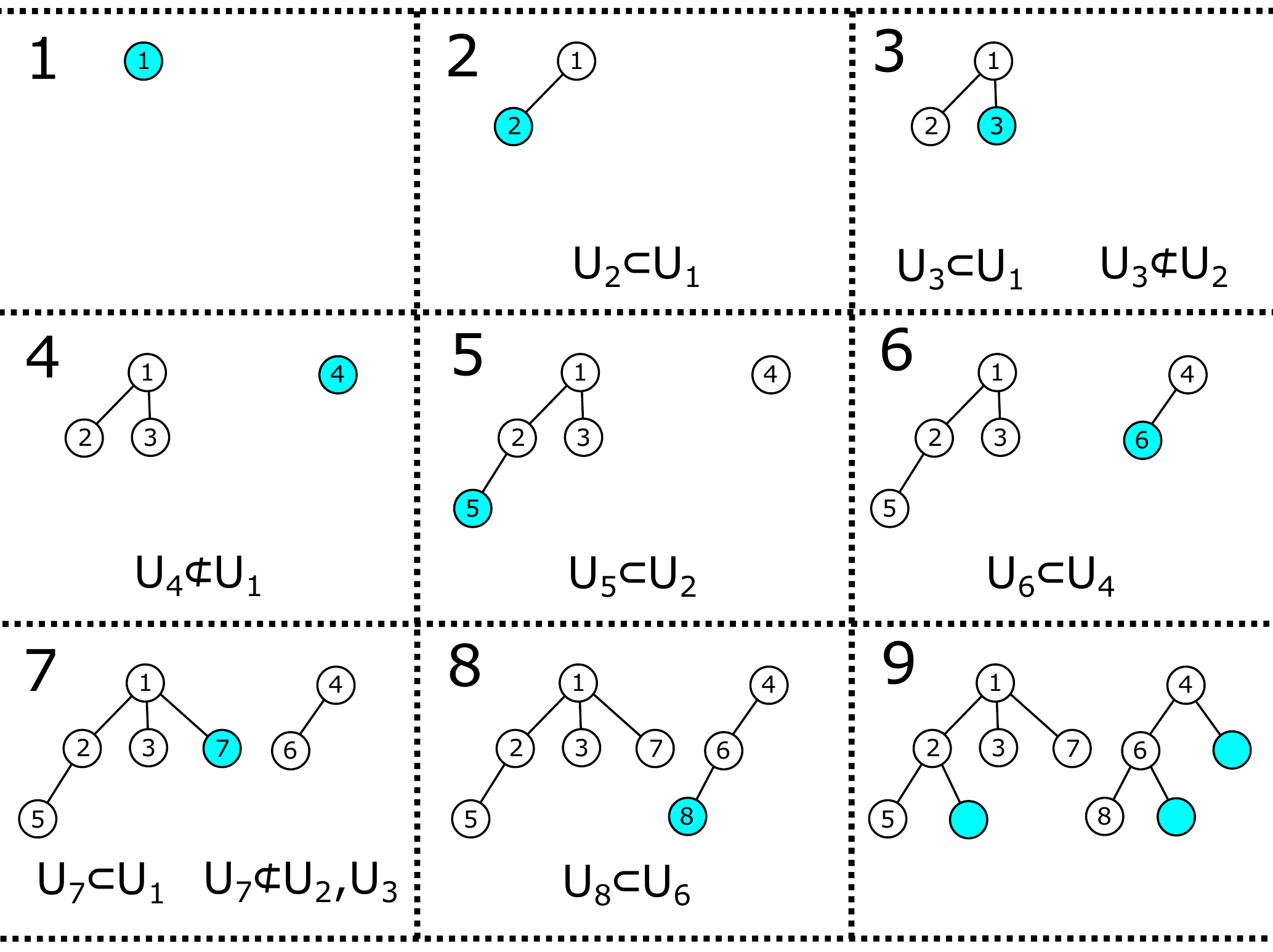}}
\caption{\label{treefig} An example of the construction of the forest structure of the sets in \algn-\uca\ with $\pfin=8$. The set corresponding to a node is a strict subset of that corresponding to its parent (if it exists) whilst the sets corresponding to nodes on different root-to-leaf paths are disjoint. For all $\lii\in[\pfin]$ the node numbered $\lii$ corresponds to the set $\les{\lii}$ and the $\lii$-th diagram depicts its construction (the blue node). The blue nodes in the 9-th diagram correspond to the sets $\les{2}\setminus\les{5}$\,,\, $\les{6}\setminus\les{8}$ and $\les{4}\setminus\les{6}$ which are all non-empty.} 
\end{figure}

\nc{\mcl}{E}
\nc{\vv}[1]{\bs{v}_{#1}}
\nc{\vc}[2]{v_{#1,#2}}
\nc{\xs}[1]{\mathcal{X}_{#1}}
\nc{\wv}[1]{\bs{w}_{#1}}
\nc{\wc}[2]{w_{#1,#2}}
\nc{\ys}[1]{\mathcal{W}_{#1}}
\nc{\ysp}{\mathcal{W}'}
\nc{\ws}{\mathcal{Y}}
\nc{\zs}[1]{\mathcal{Z}_{#1}}
\nc{\wst}[1]{\mathcal{Y}_{#1}}
\nc{\vs}{\mathcal{V}}
\nc{\bm}{\mu}
\nc{\s}{s}
\nc{\oc}{\theta}

As for the lower bound, we have the following result that proves the optimality of \algn\ in the static inventory model. 
As the static model is a special case of the dynamic one, this also proves optimality in the dynamic inventory model.
\begin{theorem}\label{thm:noisefree_lowerbound}
For any algorithm and any $M,N,C,D\in\nat$ with $\min(C,D)\leq\sqrt{\min(M,N)}$ there exists a $(C,D)$-biclustered matrix $\lmat$
of size $M\times N$, a time-horizon $T\in\nat$, and a user sequence $i_1,\ldots, i_T$ such that the algorithm has, in the static inventory model, a regret of 
\be
\Omega(\min\{C,D\}(M+N))~.
\ee
\end{theorem}

\section{The Adversarial Perturbation Case}\label{s:noisy}
%
We now turn to the problem of incorporating arbitrary perturbations of the biclustered matrix. For simplicity we will only consider the static inventory model, but note that the dynamic-inventory methodology from the perturbation-free case can be applied to this case also. Our algorithm \algnn\ only exploits similarities among items -- we leave it as an open problem to adapt our methodology in order to exploit similarities among users.

We give two algorithms, \algnn-\uie\ (\algnn\ with User Item Exclusion) and \algnn-\ue\ (\algnn\ with User Exclusion) which have expected regret bounds of
\[
\expt{\reg} = \bo{( D\icr+m+n)(M+N)}~~~~ \mbox{and}~~~~
\expt{\reg} = \bo{(D+n+m/\icr)(M+N\icr)}~,
\]
respectively. These algorithms can be fused together (into the algorithm \algnn) in the same way as we did for \algn-\uca\ and \algn-\ica, in order to obtain a best-of-both guarantee. These two algorithms
differ only by whether the instruction labelled ``\uie" in the pseudo-code of Algorithm \ref{a:orcastar} (check Step 7 therein) is included or not.

\begin{algorithm}[!h]
\vspace{0.2cm}
{\bf Input :} $ \psi \in \nat\setminus \{1\}$ \hfill\comm{The dependence on $\psi$ can be avoided -- see Section~\ref{ss:analysis}}\\
\noindent{\bf Initialization :}\qquad 
$\lel\la0$;~~~~~$\pos{\au}\la0$ for all $\au\in\us$;~~~~~$\exc, \exi \la\emptyset$;\\
%
{\bf For} $t = 1,\ldots, T$ :
\begin{enumerate}
\vspace{-0.03in}
\item $\cpo\la\pos{\ut{\tim}}$;
\vspace{-0.03in}
\item $\uco\leftarrow$ set of all items not recommended yet to $\ut{\tim}$;
\vspace{-0.03in}
\item \label{ntp-1} {\bf If} $\uco\cap\exi\neq\emptyset${\bf~~then }~select any $j_t$ from $\uco\cap\exi$; 
%
%
\vspace{-0.03in}
\item \label{ntp0} {\bf Else if} $\ut{\tim}\in\exc${\bf~~then } select any $\itt{\tim}$ from $\uco$;
%
%
\vspace{-0.03in}
\item \label{ntp1} {\bf Else If} $\cpo>0$ and $\lrel{\ut{\tim}}{\lir{\cpo}}=1$ and $\uco\cap\zi{\cpo}\neq\emptyset$ 
{\bf then :}
\begin{itemize}
\item Select any $j_t$ from $\uco\cap\zi{\cpo}$;
\item {\bf If} $\lrel{\ut{\tim}}{\itt{\tim}}=0${\bf~~then :}
\begin{itemize}
\item $\nnul{\cpo}{\itt{\tim}}\la\nnul{\cpo}{\itt{\tim}}+1$;~~~~~{\bf If} $\nnul{\cpo}{\itt{\tim}}>2\icr$ {\bf then} remove $\itt{\tim}$ from $\zi{\cpo}$;
\end{itemize}
\end{itemize}
\item \label{ntp2} {\bf Else if} $\cpo\neq\lel${\bf~~then :}
\begin{itemize}
\vspace{-0.03in}
\item {\bf If} $\lir{\cpo+1}\in\uco$ {\bf then} select $\itt{\tim}\la\lir{\cpo+1}$,\, {\bf Else} select any $\itt{\tim}$ from $\uco$;
\item $\pos{\ut{\tim}}\la\cpo+1$;
\end{itemize}
\item \label{ntp3} {\bf Else :} 
\begin{itemize}
\vspace{-0.1in}
\item Select $\itt{\tim}$ uniformly at random from $\uco$;~~~~~Select $\coin \sim $ Bernoulli($1/\icr$);
\item {\bf If} $\lrel{\ut{\tim}}{\itt{\tim}}=1$ and $\coin=0${\bf~~then :}
\begin{itemize}
\item Add $\ut{\tim}$ to $\exc$;~~~~~Only for {\bf \uiebold:} Add $\itt{\tim}$ to $\exi$;
\end{itemize} 
\item {\bf If} $\lrel{\ut{\tim}}{\itt{\tim}}=1$ and $\gamma = 1$ {\bf then :}\\ $~~~~~~\lel\la\lel+1$;~~~~~$\pos{\ut{\tim}}\la\lel$;~~~~~$\lir{\lel}\la\itt{\tim}$;~~~~~$\zi{\lel}\la\as$;~~~~~$\nnul{\lel}{\ai}\la 0\,\,\,\, \forall\ai\in\as$;
\end{itemize}
\end{enumerate}
\caption{One time recommendation algorithm for adversarial perturbation (\algnn).\label{a:orcastar}}
\end{algorithm}

\subsection{The One-Time Recommendation Algorithm \algnn}\label{ss:orcbanditstar}
\vspace{-0.05in}
\algnn\ is a modification of \algn-\ica\ from Algorithm \ref{a:orca}. Lines \ref{ntp1}, \ref{ntp2}, and \ref{ntp3} of the pseudo-code of \algnn\ (Algorithm \ref{a:orcastar}) correspond to Lines \ref{tp1}, \ref{tp2}, and \ref{tp3} of the pseudo-code of \algn. To arrive at \algnn\ we make the following changes to \algn-\ica:
\begin{itemize}
\vspace{-0.035in}
\item Although this isn't actually a fundamental change, since \algnn\ is in the static inventory model and we are only doing item clustering, the definition of a recommendable item simplifies: an item $j$ is recommendable at trial $t$ if and only if the current level $\ell$ of user $i_t$ is greater than zero, $j$ belongs to $\zi{\ell}$ and is not yet recommended to $i_t$\,, and $\lrel{\ut{\tim}}{\lir{\cpo}}=1$. The concept of a recommendable item features in Step \ref{ntp1}.
\item In \algn-\ica\ we removed an item $\ai$ from a set $\zi{\lii}$ as soon as we determined that there exists some user $\au\in\les{\lii}$ with $\lrel{\au}{\ai}=0$. In \algnn\ we only remove $\ai$ from the set $\zi{\lii}$ (in Step \ref{ntp1} of the pseudo-code in Algorithm \ref{a:orcastar}) when we determine that there exist over $2\icr$ such users. The variable $\nnul{\lii}{\ai}$ keeps track of how many of these users have been found so far. Note that the set $\les{\lii}$ is not explicitly defined in the pseudo-code of \algnn.
\item The next modification is rather counter-intuitive. In \algn-\ica\ (under the static inventory model) if $i_t$ is at level $\lel$ and there are no items in $\zi{\lel}$ 
that are not yet recommended to $i_t$, then we draw $j_t$ uniformly at random from the unrecommended items, and if $\lrel{i_t}{j_t}=1$ we create a new level. In \uie\ however, we only create a new level (in Step \ref{ntp3}) with probability $1/\icr$. Otherwise we \emp{exclude} $i_t$ and $j_t$. When a user $i$ is excluded, future items are recommended to it arbitrarily (in Step \ref{ntp0})  and trials $t$ in which $i_t=i$ no longer have an effect on the rest of the algorithm. When an item $j$ is excluded it will be recommended to all users as soon as possible (in Step \ref{ntp-1}). Excluded users and items are recorded in the sets $\exc$ and $\exi$, respectively. \ue\ differs slightly in that only users are excluded, not items.
\end{itemize}

\vspace{-0.05in}
\subsection{Analysis}\label{ss:analysis}
\vspace{-0.05in}
%
\begin{theorem}\label{t:orcastar}
Let \algnn\ be run with parameter $\psi \geq 2$ on an $M\times N$ matrix $\bs{L}$ and an arbitrary sequence of users $i_1,...,i_T$. Suppose that there exists a $(C,D)$-biclustered ground-truth matrix $\bs{L^*}$ that induces $m$ bad users and $n = n(\psi)$ bad items (recall Definition \ref{def:1}) on $\bs{L}$ . Then the expected regret of \algnn\ is upper bounded as
\begin{align*}
\expt{\reg} = \mathcal{O}\Bigl(\min\Bigl\{ (D\icr+m+n)(M+N),~ (D+n+m/\icr)(M+N\icr)\Bigl\}\Bigl)\,,
\end{align*}
the expectation being over the internal randomization of the algorithm. Note that $C$ and $D$ need not be known.
\end{theorem}
In Appendix \ref{sa:doubling_trick} we show that a standard doubling trick can remove the parameter $\psi$ in
\algnn, allowing us to achieve a regret bound that is only an $\bo{\ln(M)}$ factor off the regret bound of \algnn\ (Theorem \ref{t:orcastar}) with $\icr$ therein replaced by the optimal $\icr$ in hindsight.
\newline

\begin{proof}
Here we sketch the proof, deferring the full proof to Appendix \ref{pth3}. 
Let us first consider \uie. As in the analysis of \algn, we assume without loss of generality that for all trials $t$ there exists an item that $i_t$ likes and that has not been recommended to $i_t$ so far. Following similar steps as in the analysis of \algn, we first bound the contribution of each level to the regret, which is now $\mathcal{O}(M+\icr N)$.

Let $\fnd$ be the set of trials $t$ in which Line \ref{ntp3} of Algorithm \ref{a:orcastar} is executed and $\lrel{i_t}{j_t}=1$. Let $\fnl$ be the set of trials $t\in\fnd$ in which $\coin=1$ on trial $t$. Note that on each trial in $\fnl$ a level is created and on each trial $t$ in $\fnd\setminus\fnl$ user $i_t$ and item $j_t$ are excluded, which contributes $\mathcal{O}(M+N)$ to the total regret.

Trials $t\in\fnl$ in which $\lrep{i_t}{j_t}=0$ can hurt the algorithm by creating an unnecessary level. Given a non-bad user $i$, trials $t\in\fnd$ with $i_t=i$ and $\lrep{i_t}{j_t}=0$ happen infrequently enough that we can effectively ignore their contribution to the regret. However, given a bad user $i$, trials $t\in\fnd$ in which $\lrep{i_t}{j_t}=0$ can be more likely to happen. Yet, on such a trial, with probability $1-1/\icr$ user $i_t$ becomes excluded which implies that bad users contribute, in expectation, only $\mathcal{O}(M+N)$ to the regret. Hence we need not be concerned with trials $t\in\fnd$ with $\lrep{i_t}{j_t}=0$. In a similar fashion, one can show that bad items contribute $\mathcal{O}(M+N)$ to the overall regret.

Let a \emp{cluster} $\clus$ be a set of items that are not bad and are all equivalent in the matrix $\lmap$. We will be interested in the contribution to the regret of trials $t\in\cls:=\{t\in\fnd~|~j_t\in\clus\}$ with $\lrep{i_t}{j_{t}}=1$. Define $\fst:=\min\cls\cap\fnl$. Note that there are, in expectation, at most $\icr$ trials in $\{t\in\cls~|~t<\fst\}$, and all such trials are in $\fnd\setminus\fnl$. Due to the fact that a level is created on trial $\fst$ (if $\fst$ is finite), and that its representative item $j_{\fst}$ is not bad, we can prove that for all trials $t\in\cls$ with $t>\fst$ we have $\lrel{i_t}{j_{\fst}}=0$. Since $j_{\fst}$ is not bad and $j_{\fst}$ and $j_t$ are equivalent in $\lmap$, there can be at most $\icr$ such $t$ in which $\lrep{i_t}{j_{t}}=1$. This implies that in expectation at most one of them will be in $\fnl$. Summing up, this allows us to conclude that trials $t\in\cls$ with $\lrep{i_t}{j_{t}}=1$ have an expected contribution to \algnn's regret of 
$\mathcal{O}(\icr(M+N))$.

Taking a sum over all clusters, bad users and bad items gives us the required regret bound. In \ue\ we can, without loss of generality, assume that there are no bad items. Noting the fact that each trial in $\fnd\setminus\fnl$ now only contributes $\mathcal{O}(N)$, the same analysis as \uie\ gives us the desired regret guarantee.
\end{proof}

\vspace{-0.05in}
\section{Experiments}\label{s:experiments}
\vspace{-0.03in}
We now report the results of a preliminary set of experiments comparing (versions of) \algn\  to common CF baselines.

\noindent{\bf Datasets.} 
The MovieLens dataset, produced by the GroupLens research team (\url{https://grouplens.org/datasets/movielens}) \citep{harper2015movielens}, is commonly used in recommender system studies.
It consists of 1M ratings in the range [1, 5], and we follow the convention used in previous research~(e.g., \citet{lim2015top}), where ratings above 3 are considered as positive feedback.


To test the algorithms under different conditions, we selected subsets of Movielens with increasing number of items. We selected $N$ items uniformly at random, with $N \in \{50, 100, 200\}$. Then we retained all the ``likes'' where the item is in $N$. The resulting average number of users $M$ and likes are
$3376 \pm 535$ and $7765 \pm 1870$ for $N= 50$ items,
$4488 \pm 301$ and $15886 \pm 2375$ for $N=100$,
$5290 \pm 191$ and $32123 \pm 4117$ 
for $N=200$.
In each experiment, to be sure we go through all likes, we ran $M \times N$ rounds of recommendations.

\begin{figure}[t!]
\vspace{-0.15in}
    \includegraphics[width=0.3\textwidth]{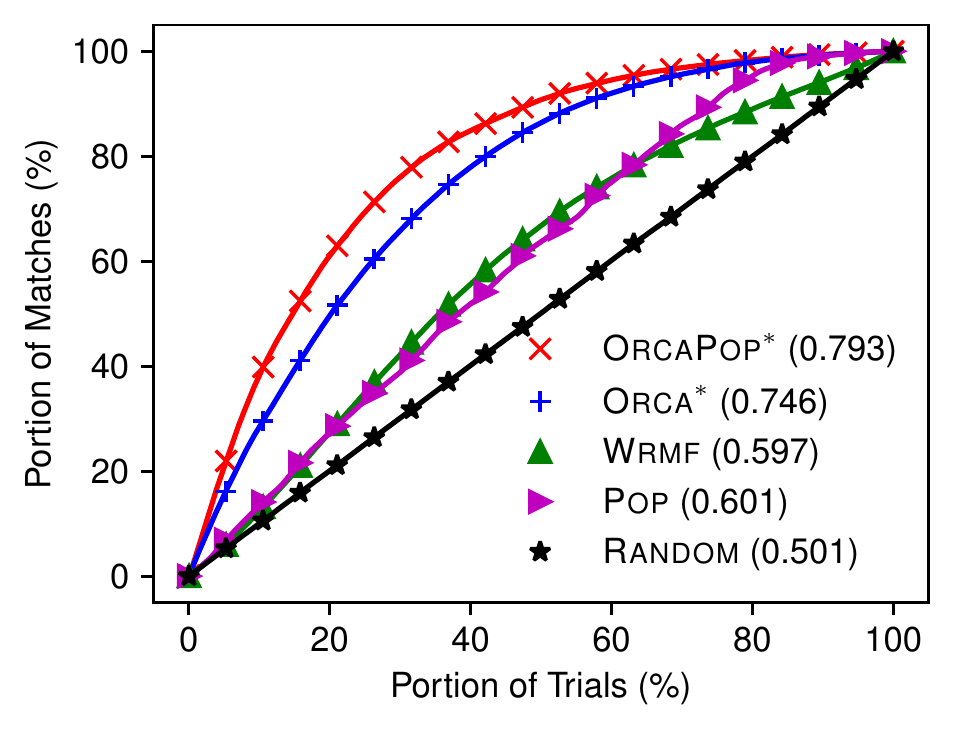}
    \hspace{0.12in}
    \includegraphics[width=0.3\textwidth]{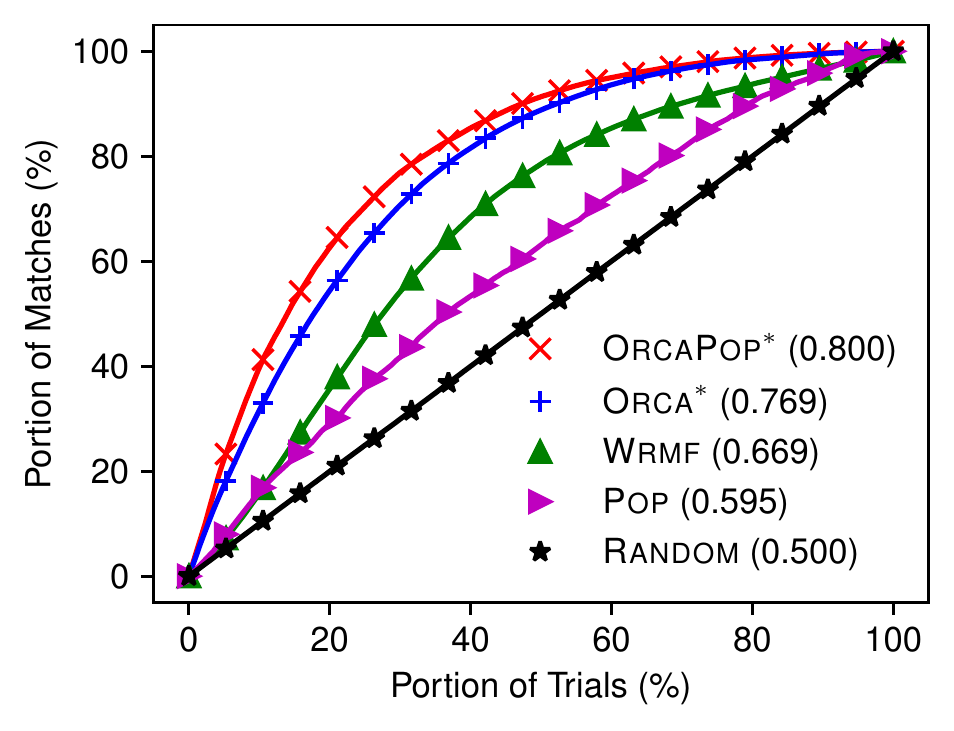}
    \hspace{0.12in}
    \includegraphics[width=0.3\textwidth]{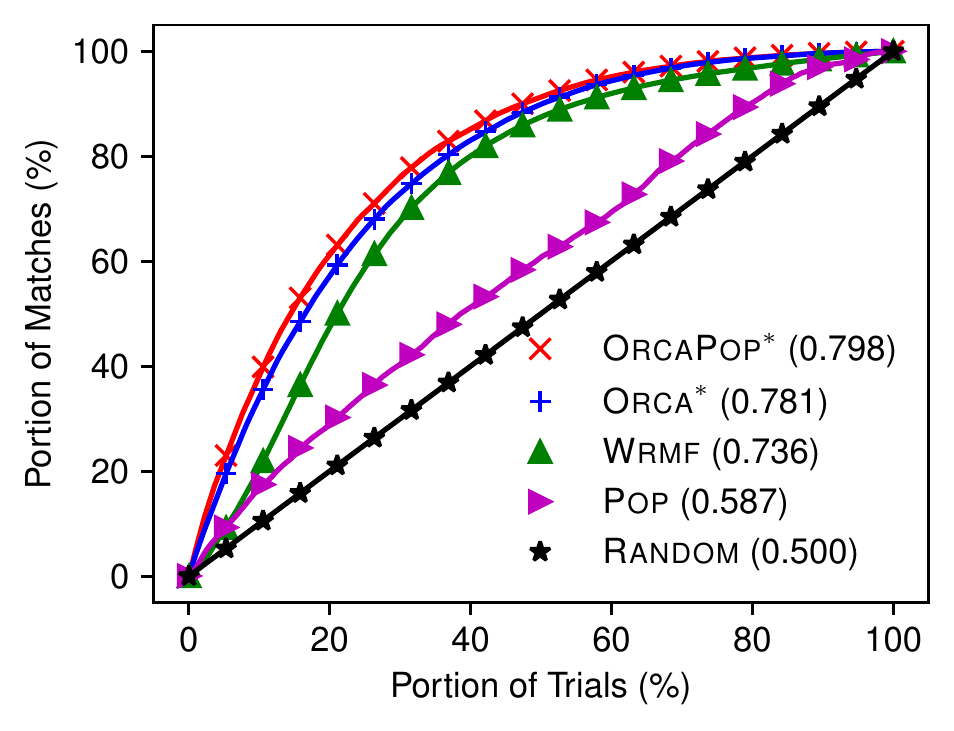}
    \vspace{-0.12in}
    \caption{Recommendation curves for subsets of MovieLens with varying number of items (from left to right: $N=50$, $N=100$, $N=200$ items). 
    The curves are averages across 30 repetitions. The numbers in braces give the area under each curve (the higher the better). 
    The displayed curves for \wrmf\ are the best performers across the number of latent factors.}
    \vspace{-0.3in}
    \label{ffig:results} 
\end{figure}

\noindent{\bf Algorithms and baselines.}
We use as baselines two naive methods and one based on Matrix Factorization (MF).
The first method, referred to as \random, just recommends an item at random (among the available ones for the current user). The second method, \pop\ (``popularity"), leverages the fact that, in MovieLens, item popularity is highly predictive of a successful recommendation, and we recommend the most popular among the items that have been recommended in the past.
Finally, as an MF approach, we used a well-known Matrix Factorization algorithm, specifically the Weighted Regularized Matrix Factorization (\wrmf) proposed by~\cite{koren:icdm08,pan+08}.
Since it was originally designed for batch recommendation, we adapted \wrmf\ to an online setting as follows. Initially, random recommendations are generated for the first 1,000 rounds, as \wrmf\ requires a minimum amount of ``likes" to perform MF. Subsequently, we fit the \wrmf\ model to produce recommendations for the next batch of 1,000 rounds. As the number of rounds increases, the model needs to be trained less frequently, as more data is needed to obtain a significant performance increase. Hence, to make this approach feasible for millions of rounds, the batch size is incrementally increased
by 10\% at each re-train phase.
This adjustment 
causes only a negligible difference in performance compared to full batch training at each round.
We relied on the 
\textsc{mrec} recommender system library (\url{https://mendeley.github.io/mrec/}), using default values for confidence weight ($1$) and regularization constant ($0.015$), and 15 iterations of alternating least squares, since performance was less sensible to these parameters. Instead, the number of latent factors
turned out to be very important for actual performance. We report the results for \wrmf\ with 4, 8, 16, and 32 latent factors.

We compared these baselines to\footnote
{
We decided not run \algn\ here, as the preference matrix is not $(C,D)$-biclustered (for some small $C$ and $D$). 
} 
\algnn.
Moreover, in order to show the versatility of \algnn, we also leverage the fact that popular items have higher probability of success, and in steps 3, 4, 5, 6 we replace \emph{select any $j_t$ from $\square$} with \emph{select the most popular $j_t$ from $\square$}. 
This modification does not affect the validity of our theoretical results. We call the resulting algorithm \orcapop*.
%
In all experiments, in order to avoid unfair comparisons with the baselines, the user sequence $i_1,\ldots, i_T$ is always generated uniformly at random over all available users.

\noindent{\bf Results. } 
All our experiments have been run on an Intel Xeon Gold 6312U - 24c/48t - 2.4 GHz/3.6 GHz with 256 GB ECC 3200 MHz memory.
%
Our results are contained in Figure \ref{ffig:results} and Tables \ref{tab:addlabel} and \ref{tab:addlabel2} (Appendix \ref{sa:further}). Figure \ref{ffig:results} plots, for each of the three values of $N$, the fraction of uncovered user-item matches (i.e., the ``likes" in the preference matrix $\lmat$) against the fraction of recommendations so far. The total number of recommendation (100\% in the $x$-axis) always corresponds to $M\times N$, that is, the total number of entries in $\lmat$. These curves quantify the pace at which the matches are uncovered by the different algorithms, hence the higher the better. A more compact version of this metric is the area under these curves, reported in both Figure \ref{ffig:results} and Table \ref{tab:addlabel}.

Though a bit preliminary in nature, some trends emerge:
The performance of \algnn\ and \orcapop* is consistently high even at small number of items, where MF algorithms suffer more due to the cold start problem.
At higher number of items, when cold start is less evident, \wrmf\ and \algnn\ have closer performances.
The \pop\ heuristic suffers from the subsampling of the items we made during preprocessing.

\vspace{-0.1in}
\section{Conclusions and Limitations}
\vspace{-0.05in}
We have considered a sequential content recommendation problem where items can only be recommended once to each user (``no-repetition constraint"). Unlike the abundant literature on MAB under low-rank and clustering assumptions, we have handled through a regret analysis the more general situation where users show up in the system in an arbitrary (possibly adversarial) order. We have proposed \algn, that works under biclustering assumptions, and have shown that this algorithm exhibits an optimal (up to constant factor) regret guarantee against an omniscient oracle that knows the user-item preference matrix ahead of time. 
%
We have then extended \algn\, to \algnn\,, a more robust version which is able to handle 
arbitrary preference matrices.
Finally, we have provided preliminary empirical evidence of the effectiveness of (versions of) our algorithm as compared to standard baselines.

Currently, \algnn\, is not able to leverage similarities among users. This is likely to require a substantial redesign of our algorithm.
%
Among the relevant extensions that would allow us to better address real-world scenarios are: i. The case where the learner has access to side information (i.e., features) about users and/or items, 
ii. The case where the user feedback is non-binary (e.g., relevance scores rather than clicks), and iii. Extending the biclustering assumption to the slightly more general low-rank assumption, ubiquitous in the CF literature.





\bibliography{ORCbibliography}

\appendix

\crefalias{section}{appendix} 

\section{Further Related Work}\label{sa:related}

Our problem also shares similarities with the classical problem of {\em Matrix Completion} (MC)~
see, e.g., the papers by~\cite{keshavan2010matrix, biau2010statistical, rohde2011estimation, aditya2011channel, candes2012exact, negahban2012restricted, jain2013low}. 
The objective in a MC problem is to estimate the remaining entries of a given matrix having at one's disposal a subset of the observed entries. It is often assumed that the matrix meets specific criteria (like low rank). Again, a closer inspection reveals that the typical conditions under which a MC algorithm work do not properly reflect the sequence of user events in a RS. For instance, whereas MC algorithms typically require the subset of the entries to be drawn at random according to some distribution, we are not bound to observe the matrix entries according to such a benign criterion, for users may visit a RS and revisit in an arbitrary order.
Furthermore, we impose differing structural assumptions, such as the adversarial perturbation of the preference matrix.
The closest MC problem to ours is when the components need to be predicted online~\citep{ hazan2012near, herbster2020online}. However, this problem is fundamentally different from ours in that on each trial a component of the preference matrix must be predicted instead of an item selected to be recommended to a given user. 






Matrix factorization and structured bandit formulations have often been used to frame the design of RS algorithms -- see, e.g., \cite{koren2009matrix, dabeer2013adaptive, zhao2013interactive,wang2006unifying, verstrepen2014unifying}, and references therein. The relevant literature on RS is abundant, and we can hardly do it justice here. Yet, we observe that most of the classical investigations on RS are experimental in nature, while those on structural bandits (e.g., \cite{bui2012clustered,maillard2014latent, gentile2014online,10.1145/2911451.2911548, kwon2017sparse, jedor2019categorized,katariya+17,10.1145/3240323.3240408,NEURIPS2020_9b7c8d13,pmlr-v97-jun19a,pmlr-v117-trinh20a,pmlr-v130-lu21a,kang2022efficient,pal2023optimal} or on MC \cite{keshavan2010matrix, biau2010statistical, rohde2011estimation, aditya2011channel, candes2012exact, hazan2012near, negahban2012restricted, jain2013low, herbster2020online} do not readily apply to our adversarial scenarios.

\section{Proofs}
This appendix contains the complete proofs of all our claims.

\subsection{Proof of Theorem \ref{thm:noisefree}}\label{pth2}
\begin{proof}
We shall show that \uca\ and \ica\ have expected regret bounds of 
\be
\expt{\reg}=\mathcal{O}(\uc(\nmu+\nma))~~~\operatorname{and}~~~\expt{\reg}=\mathcal{O}(\ic(\nmu+\nma))~,
\ee 
respectively. 

Let us assume that on every trial $\tim$ there exists an item $\ai\in\ist{\tim}$ which $\ut{\tim}$ likes and has not been recommended to $\ut{\tim}$ before. This is without loss of generality since on any trials in which this does not hold the omniscient oracle (when suggesting liked items before disliked items) is forced to make a mistake.

Let $\pfin$ be the value of $\lel$ on trial $\ntr$. Given a level $\lii\in\na{\pfin}$ we will bound the number of mistakes made on trials of the following types:
\begin{itemize}
\item Trials corresponding to Line \ref{tp1} (of the pseudocode in Algorithm \ref{a:orca}) with $\lse=\lii$. Given such a trial $\tim$, we have $\itt{\tim}\in\zi{\lii}$ and if a mistake is made $\itt{\tim}$ is removed from $\zi{\lii}$ so only one mistake can be made per item. Hence, no more than $\nma$ mistakes are made on such trials.
\item Trials corresponding to Line \ref{tp2} with $\cpo=\lii-1$. There is at most one such trial per user and hence no more than $\nmu$ mistakes are made on such trials.
\item Trials corresponding to Line \ref{tp3} with $\cpo=\lii-1$. On such a trial $\tim$, item $\itt{\tim}$ is selected uniformly at random from $\uco$. Since (by the initial assumption) there exists an item $\ai\in\uco$ with $\lrel{\ut{\tim}}{\ai}=1$, there are in expectation at most $\nma$ such trials $\tim$ until $\lrel{\ut{\tim}}{\itt{\tim}}=1$. Once this happens, there are no more trials of this type, and hence there are at most $\nma$ such trials in expectation.
\end{itemize}
Thus there are in expectation $\mathcal{O}(\nmu+\nma)$ mistakes in trials of the above types for each $\lii\in\na{\pfin}$, implying 
$$
\expt{\reg}=\mathcal{O}(\expt{\pfin}(\nmu+\nma))~.
$$
Therefore, all we now need to prove is that $\pfin=\mathcal{O}(\uc)$ and $\pfin=\mathcal{O}(\ic)$ for \uca\ and \ica, respectively.

To this effect, recall that given a level $\lii\in[\pfin]$ we denote by $\cru{\lii}$ the user that created that level (i.e. $\cru{\lii}:=i_t$ when $t$ is the trial on which level $\lii$ was created).

A crucial property needed to prove this is what we call the \emp{separation property}:

\begin{center}
Given $\lii,\lij\in\na{\pfin}$ with $\lij>\lii$ and $\cru{\lij}\in\les{\lii}$, there exists some $\au\in\les{\lii}$ with $\lrel{\au}{\lir{\lij}}=0$. 
\end{center}

To see why this property holds, first let $\tim$ be the trial on which $\cru{\lij}$ creates level $\lij$ (so $\cru{\lij}=\ut{\tim}$). Since $\lii<\lij$, on such round $\tim$ we have, direct from the algorithm, that either $\cru{\lij}\notin\les{\lii}$ or there is no item in $\zi{\lii}\cap\ist{\tim}$ that has not yet been recommended to $\cru{\lij}$. So since $\cru{\lij}\in\les{\lii}$, and since $\lir{\lij}$ is recommended to $\cru{\lij}$ on trial $\tim$ and hence not recommended to $\cru{\lij}$ before, we must have that $\lir{\lij}\notin\zi{\lii}$ on trial $\tim$. But for this to happen there must exist a user $\au\in\les{\lii}$ with $\lrel{\au}{\lir{\lij}}=0$, as claimed.

Now let the symbol $\eqi$ denote that two users or items are equivalent with respect to the matrix $\lmat$.

Let us first focus on \ica. Suppose, for contradiction, that we have $\lii,\lij\in\na{\pfin}$ with $\lij>\lii$ and $\lir{\lij}\eqi\lir{\lii}$. Since $\lrel{\cru{\lij}}{\lir{\lij}}=1$ we also have $\lrel{\cru{\lij}}{\lir{\lii}}=1$, and hence $\cru{\lij}\in\les{\lii}$. This implies, via the separation property, that there exists $\au\in\les{\lii}$ with $\lrel{\au}{\lir{\lij}}=0$, so choose such a $\au$. Since $\au\in\les{\lii}$ this gives $\lrel{\au}{\lir{\lii}}=1$, which contradicts the fact that $\lir{\lii}\eqi\lir{\lij}$. So all the representatives of the different levels come from different clusters which gives us $\pfin\leq\ic$, as required.

We now turn our attention to \uca. We will show that for all $\lii,\lij\in\na{\pfin}$ with $\lij>\lii$ we have either $\les{\lij}\cap\les{\lii}=\emptyset$ or $\les{\lij}\subset\les{\lii}$, where the subset property is strict (i.e., $\les{\lij}\neq\les{\lii}$). To show this, suppose that $\les{\lij}\cap\les{\lii}\neq\emptyset$. Choose $\au\in\les{\lij}\cap\les{\lii}$. Since $\au\in\les{\lii}$ we have $\lrel{\au}{\lir{\lii'}}=\lrel{\cru{\lii}}{\lir{\lii'}}$ for all $\lii'\in\na{\lii}$, and since $\au\in\les{\lij}$ we have $\lrel{\au}{\lir{\lii'}}=\lrel{\cru{\lij}}{\lir{\lii'}}$ for all $\lii'\in\na{\lij}$. Thus, since $\lij>\lii$ we have $\lrel{\cru{\lij}}{\lir{\lii'}}=\lrel{\au}{\lir{\lii'}}=\lrel{\cru{\lii}}{\lir{\lii'}}$ for all $\lii'\in\na{\lii}$, which in turn implies that $\cru{\lij}\in\les{\lii}$. Hence, for all $\au\in\les{\lij}$ and $\lii'\in\na{\lii}$ we have $\lrel{\au}{\lir{\lii'}}=\lrel{\cru{\lij}}{\lir{\lii'}}=\lrel{\cru{\lii}}{\lir{\lii'}}$ so that $\au\in\les{\lii}$. This implies that $\les{\lij}\subseteq\les{\lii}$. Hence, since $\cru{\lij}$ is trivially contained in $\les{\lij}$ it is also contained in $\les{\lii}$ which implies, by the separation property, that there exists a user $\au\in\les{\lii}$ with $\lrel{\au}{\lir{\lij}}=0$. Consider such an $\au$. Since (directly from the algorithm) we have $\lrel{\cru{\lij}}{\lir{\lij}}=1$ we then have $\lrel{\au}{\lir{\lij}}\neq\lrel{\cru{\lij}}{\lir{\lij}}$, so that $\au\notin\les{\lij}$. Hence $\les{\lij}\subseteq\les{\lii}$, and there exists $\au\in\les{\lii}\setminus\les{\lij}$ which implies that $\les{\lij}\subset\les{\lii}$, as required.

We have just shown that for all $\lii,\lij\in\na{\pfin}$ with $\lij>\lii$ we have either $\les{\lij}\cap\les{\lii}=\emptyset$ or $\les{\lij}\subset\les{\lii}$. We call this property the \emp{tree property}. We will now construct a directed graph (see Figure \ref{treefig} for an example), whose nodes are sets, as follows. For all $\lij\in\na{\pfin}$ we have that $\les{\lij}$ is a node in the graph and that:
\begin{itemize}
\item If there exists $\lii\in\na{\pfin}$ with $\les{\lij}\subset\les{\lii}$ then the (unique) parent of $\les{\lij}$ is $\les{\lii}$ for the maximum such $\lii$;
\item If there does not exist such a level $\lii$ then $\les{\lij}$ has no parent.
\end{itemize}

Note that, by the tree property, if $\les{\lii}$ is the parent of $\les{\lij}$ then (since $\les{\lij}\subset\les{\lii}$) we have $\lii<\lij$, thereby making the graph acyclic. Moreover, since each node has at most one parent the graph is a forest.

Suppose we have $\lii,\lij\in[\pfin]$ such that $\lij>\lii$ and that $\les{\lii}$ and $\les{\lij}$ are both roots. Since $\les{\lii}$ is a root we must have that $\les{\lii}\not\subset\les{\lij}$ so by the tree property  $\les{\lii}\cap\les{\lij}=\emptyset$ holds. Now suppose we have $\lii,\lij\in[\pfin]$ such that $\lij>\lii$ and that $\les{\lii}$ and $\les{\lij}$ are siblings. Let $\lii'$ be such that $\les{\lii'}$ is the parent of these siblings. We must have $\les{\lii}\subset\les{\lii'}$ so, again by the tree property, we have $\lii'<\lii$. We also must have that $\lii'$ is the maximum element $\lij'$ of $[\pfin]$ such that $\les{\lij}\subset\les{\lij'}$. These two properties imply that $\les{\lij}\not\subset\les{\lii}$ and hence, by the tree property,  $\les{\lii}\cap\les{\lij}=\emptyset$. This shows that any two roots, and any two siblings, correspond to disjoint sets.

For all $\lii,\lij\in[\pfin]$ such that $\les{\lij}$ is the only child of $\les{\lii}$, create a new node $\les{\lii}\setminus\les{\lij}$ and make it a child of $\les{\lii}$. Since, here, $\les{\lij}\subset\les{\lii}$, such new nodes are non-empty, and hence, because for all $\lii\in\na{\pfin}$ we have $\cru{\lii}\in\les{\lii}$, all nodes are non-empty. Note that the property that any pair of siblings or roots are disjoint still holds so, since any node is a subset of each of its ancestors, all leaves are disjoint. Also, for all $\lii\in\na{\pfin}$ and $\au\in\les{\lii}$ we have $\au'\in\les{\lii}$ for all $\au'\in\as$ with $\au'\eqi\au$. This implies that each leaf of the forest contains a user cluster as a subset and hence that $\uc$ is at least as large as the number of leaves. Since all the internal nodes of the forest have at least two children and the number of nodes in the forest is no less than $\pfin$, we must have at least $\pfin/2$ leaves. Hence $\pfin\leq2\uc$, as required. 
\end{proof}

\subsection{Proof of Theorem \ref{thm:noisefree_lowerbound}}\label{pth3}
\begin{proof}
Let $\mcl:=\min\{C,D\}$. We will construct our $M\times N$ matrix $\lmat$ such that it is both $\mcl$-user clustered and $\mcl$-item clustered, which implies $\lmat$ is also $(C,D)$-biclustered.

First consider the case that $M\geq N$. Without loss of generality, assume that $N$ is a multiple of $\mcl$. For all $a\in[\mcl]$ let $\vv{a}$ be the $N$-component vector such that for all $j\in[N]$ we have $\vc{a}{j}:=1$ if and only if 
$$
(a-1)(N/\mcl)<j\leq aN/\mcl~.
$$
Define $T:=M\mcl$ and for all $i\in[M]$ and $t\in[T]$ with $(i-1)\mcl<t\leq i\mcl$, let $i_t:=i$. For all $i\in[M]$ we will choose some $a_i\in[\mcl]$ in a way that is dependent on the algorithm and set the $i$-th row of $\lmat$ equal to $\vv{a_i}$. Note that we can always choose $a_i$ such that in expectation $\Omega(\mcl)$ mistakes are made in the $\mcl$ trials $t$ for which $i_t=i$. Since $N/E\geq E$, an omniscient oracle would make no mistakes, and hence the expected regret of the learner is equal to its expected number of mistakes, which is $\Omega(M\mcl)$ as required.

We now turn to the case that $N\geq M$. Without loss of generality, assume that $N$ is a multiple of $\mcl$ and assume $M=\mcl^2$ (since for any $i\in[M]$ with $i>\mcl^2$ we will be able to choose the $i$-th row of $\lmat$ arbitrarily).
For all $a\in[\mcl]$ define $\xs{a}$ to be the set of all $i\in[M]$ such that 
$$
(a-1)\mcl<i\leq a\mcl~.
$$
We will construct our matrix $\lmat$ so that for all $a\in[\mcl]$ there exists an $N$-component vector $\wv{a}$ such that for all $i\in\xs{a}$ the $i$-th row of $\lmat$ is equal to $\wv{a}$. Note that $\lmat$ will then be $\mcl$-user clustered as required. We set our time horizon $T:=NE$. Our user sequence is defined as follows. For all $i\in[M]$ we have that $i_t:=i$ for all $t\in[T]$ with 
$$
(i-1)N/\mcl<t\leq iN/\mcl~.
$$
For all $a\in[\mcl]$, let $\zs{a}$ be the set of trials $t\in[T]$ with $i_t\in\xs{a}$, noting that $|\zs{a}|=N$ and the trials in $\zs{a}$ come directly before those of $\zs{a+1}$.

We now turn to the construction of the vectors $\{\wv{a}~|~a\in[\mcl]\}$. To do so, we will construct, in order, a sequence of sets $\{\ys{a}~|~a\in[\mcl]\}$ where, for all $a,a'\in[\mcl]$ with $a'\neq a$, we have $\ys{a}\subseteq[N]$ and $|\ys{a}|=N/\mcl$ and $\ys{a'}\cap\ys{a}=\emptyset$. 

For all $a\in[\mcl]$ the vector $\wv{a}$ is defined so that for all $j\in[N]$ we have 
$$
\wc{a}{j}:=\indi{j\in\ys{a}}~.
$$
Suppose we have constructed $\ys{a'}$ for all $a'$ less than some $a\in[\mcl]$. Take an arbitrary set $\ysp\subseteq[N]$ with $|\ysp|=N/\mcl$ and $\ysp\cap\ys{a'}=\emptyset$ for all $a'\in[a-1]$. Suppose that $\ys{a}$ is set equal to $\ysp$ and the learning algorithm is run. Given some trial $t\in\zs{a}$, let $\wst{t}$ be the set of all items $j\in\ysp$ such that there exists a trial $t'\in\zs{a}$ with $t'<t$ and $j_{t'}=j$. Let $\s$ be the first trial in $\zs{a}$ in which 
$$
|\wst{\s}|=N/(4\mcl)
$$ 
(or $\max\zs{a}$ if no such $\s$ exists), and let $\vs$ be the set of all $t\in\zs{a}$ with $t<\s$.
The only trials $t\in\vs$ in which the algorithm (with knowledge of $\ys{a'}$ for all $a'\in[a-1]$) can be assured of not making a mistake are contained in the set of trials $t'\in\vs$ such that there exists an item $j\in\wst{\s}$ that has not been recommended to $i_{t'}$ before trial $t'$. Since $|\wst{\s}|\leq N/(4\mcl)$ and there are only $\mcl$ users $i$ in which $i_t=i$ for some $t\in\zs{a}$, there are at most $N/4$ trials in $\vs$ in which the algorithm is assured of not making a mistake. Similarly there are at most $N/4$ trials in $\vs$ in which no mistakes are made. As we shall see, $N/4$ is small enough that we can ignore such trials. On all other trials in $\vs$ the algorithm must search for an item in $\ysp$. Since $\ysp$ is an arbitrary subset of $N-(a-1)N/\mcl$ elements with cardinality $N/\mcl$ we can choose $\ysp$ in such a way that there are, in expectation,
$$
\Omega\Bigl(|\wst{\s}|(N-(a-1)N/\mcl)/(N/\mcl)\Bigl)=\Omega(N(\mcl-a)/\mcl)
$$ 
such trials in which mistakes are made. This is because (from above) there are at most $N/4$ trials in $\vs$ in which mistakes are not made, and there exists a constant $\oc$ such that 

$$
\oc N(\mcl-a)/\mcl+N/4 \leq N~,
$$ 
$N$ being the number of trials in $\zs{a}$.
 
Summing the above mistake lower-bounds over all $a\in[\mcl]$ gives us a total mistake bound of $\Omega(N\mcl)$. Since, for all $a\in[\mcl]$, we have $|\ys{a}|=N/\mcl$, and each user is queried $N/\mcl$ times, an omniscient oracle would make no mistakes so the expected regret is equal to the expected number of mistakes which is $\Omega(N\mcl)$ as claimed. Moreover, for any item $j\in[N]$, any $a\in[\mcl]$ and any user $i\in\xs{a}$ we have $\lrel{i}{j}=\indi{j\in\ys{a}}$ so since the sets $\{\ys{a}~|~a\in[\mcl]\}$ partition $[N]$ we have that $\lmat$ is $\mcl$-item clustered as required.
\end{proof}

\subsection{Proof of Theorem \ref{t:orcastar}}
\begin{proof}
We will first analyze \uie\ and then show how to modify the analysis for \ue.

Recall that for all users $i\in[M]$ the values $\nqu{i}$ and $\nli{i}$ are the number of times that user $i$ is queried and the number of items that user $i$ likes, respectively. We can assume without loss of generality that $\nqu{i}\leq\nli{i}$ for all users $i\in\na{M}$, so that the regret is the number of mistakes. This is because if, on some trial $t$, there is no item that $i_t$ likes and has not been recommended to $i_t$ so far, then on such a trial the omniscient oracle (when suggesting liked items before disliked items) would incur a mistake. Note that this assumption means that on every trial $t$ there exists an item $j$ that $i_t$ likes and has not been recommended to $i_t$ so far. This assumption also entails that the regret of \algnn\, is equal to its number of mistakes.

Let $\fnd$ be the set of trials $t$ in which Line \ref{ntp3} of the pseudocode in Algorithm \ref{a:orcastar} is invoked and $\lrel{i_t}{j_t}=1$. Let $\fnl$ be the set of trials $t\in\fnd$ in which $\coin=1$ on trial $t$. Note that on each trial in $\fnl$ a level is created, and hence $|\fnl|=\pfin$ where $\pfin$ the value of $\lel$ on the final trial $\ntr$.

We will now bound the expected number of mistakes in terms the cardinality of the above sets. To do this, we consider a trial $t$ in which a mistake is made. We have the following possibilities on trial $t$:
\begin{itemize}
\item The condition in Line \ref{ntp-1} of Algorithm \ref{a:orcastar} is true. In this case $j_t\in\exi$. For all $j\in\exi$ we have that $j$ was added to $\exi$ on some trial in $\fnd$, and we know that the number of rounds $t$ in which $j_t=j$ is bounded from above by $M$. Hence there can be at most $M|\fnd|$ such trials $t$.

\item The condition in Line \ref{ntp0} holds. In this case $i_t\in\exc$. For all $i\in\exc$ we have that $i$ was added to $\exc$ on some trial in $\fnd$, and we know that the number of trials $t$ in which $i_t=i$ is bounded from above above by $N$. Hence there can be at most $N|\fnd|$ such trials $t$.

\item The condition in Line \ref{ntp1} holds. Let $j:=j_t$ and let $\ell$ be the value of $\ell_{i_t}$ on trial $t$. We must have that $j_t\in\zi{\ell}$, and hence that 
$$
\nnul{\cpo}{\itt{\tim}}\leq 2\icr
$$
at the start of trial $t$. But $\nnul{\cpo}{\itt{\tim}}$ is increased by one on such a trial, which means there can be no more than $2\icr$ trials $t$ with $j_t=j$ and $\ell_{i_t}=\ell$. Note that each level $\ell\in[\pfin]$ is created on a trial in $\fnl$ so there are at most 
$$
2\icr N|\fnl|
$$ 
mistakes made on trials in which Line \ref{ntp1} applies.

\item The condition in Line \ref{ntp2} is true. For every level $\ell\in[\pfin]$ and every user $i\in[M]$ there is at most one such trial $t$ with $i_t=i$ and $\ell_{i}=\ell$ (on trial $t$). Hence there are no more than $M\pfin$ such trials in total. Note that each level $\ell\in[\pfin]$ is created on a trial in $\fnl$ so there are at most 
$$
M|\fnl|
$$
mistakes made on trials in which Line \ref{ntp2} applies.

\item The condition in Line \ref{ntp3} holds. Suppose $t'$ is a trial in which the condition in Line \ref{ntp3} is true but $\lrel{i_{t'}}{j_{t'}}=1$. This means that $t'\in\fnd$, so there cannot be more than $|\fnd|$ such trials $t'$. But given an arbitrary trial $t'$ in which that condition holds, the probability that $\lrel{i_{t'}}{j_{t'}}=1$ is at least $1/N$ (since $\nqu{i_{t'}}\leq\nli{i_{t'}}$). This implies that there are, in expectation, at most 
$$
N\expt{|\fnd|}
$$ 
trials in which Line \ref{ntp3} applies and a mistake is made. 
\end{itemize}

Putting together, we have so far shown that:
\be
\expt{\reg} = \bo{(M+\icr N)\expt{|\fnl|}+(M+N)\expt{|\fnd|}}~.
\ee 

Recall that given a user $i\in[M]$, its perturbation level $\ucp{i}$ is the number of items $j\in[N]$ in which $\lrel{i}{j}\neq\lrep{i}{j}$.
Given a trial $t\in\na{T}$, let $\rem{t}$ be the number of items that user $i_t$ likes and have not been recommended to them so far. Let $\fnb$ be the set of trials $t\in\fnd$ with $\lrep{i_t}{j_t}=0$ and $\rem{t}>2\ucp{i_t}$, and $\fnc$ be the set of trials $t\in\fnd$ with $\lrep{i_t}{j_t}=0$ and $\rem{t}\leq2\ucp{i_t}$. 

Let $\gi$ be the set of items which are \emp{good} (i.e., not bad). We call a non-empty set $\clus\subseteq\gi$ a \emp{cluster} if and only if for all pairs of items $j,j'\in\clus$ we have that the $j$-th and $j'$-th columns of $\lmap$ are identical and, in addition, for all items $j''\in\gi\setminus\clus$, the $j$-th and $j''$-th columns of $\lmap$ differ. Note that there are no more than $D$ clusters. Given a cluster $\clus$, we define $\cls$ to be the set of all $t\in\fnd$ with $j_t\in\clus$ and such that $i_t$ is \emp{good} (that is, not bad). 

Let us now focus on a specific cluster $\clus$ and define 
$$
\fst:=\min(\cls\cap\fnl)~,
$$ 
with the convention that the minimizer of the empty set is $\infty$. Let $\fsl$ be the level created on trial $\fst$. We then partition $\cls$ into the following sets:
\begin{itemize}
\item $\cl{1}$ is the set of all $t\in\cls$ with $t<\fst$;
\item $\cl{2}$ is the set of all $t\in\cls$ with $t\notin\fnb\cup\fnc\cup\fnl$ and $t\geq\fst$;
\item $\cl{3}$ is the set of all $t\in\cls\cap\fnl$ with $t\notin\fnb\cup\fnc$ (noting this implies $t\geq\fst$);
\item $\cl{4}$ is the set of all $t\in\cls\cap\fnb$ with $t\geq\fst$;
\item $\cl{5}$ is the set of all $t\in\cls\cap\fnc$ with $t\geq\fst$.
\end{itemize}

We will next analyze how much each of these sets contributes to the above regret bound.

Every $t\in\fnd$ has a $1/\icr$ probability of being in $\fnl$, which implies that the expected cardinality of $\cl{1}$ is at most $\icr$. Since each element of $\cl{1}$ is not in $\fnl$, it contributes $\bo{M+N}$ to the regret bound. Hence, the overall contribution of $\cl{1}$ to the regret bound is in expectation equal to 
$$
\bo{\icr (M+N)}~.
$$

We will now show that for all $j\in\clus$ we always have 
$$
\nnul{\fsl}{j}\leq2\icr~.
$$ 
To see this, take such a $j$ and suppose we have a round $t\in[T]$ in which $\nnul{\fsl}{j}$ is incremented. Note that on such a $t$ we necessarily have $\lrel{i_t}{\lir{\fsl}}=1$ and $\lrel{i_t}{j}=\lrel{i_t}{j_t}=0$. We have the following two possibilities:
\begin{itemize}
\item If $\lrep{i_t}{\lir{\fsl}}=0$ then since $\lrel{i_t}{\lir{\fsl}}=1$ we have $\lrep{i_t}{\lir{\fsl}}\neq\lrel{i_t}{\lir{\fsl}}$. Since $\lir{\fsl}=j_{\fst}$ and $\fst\in\cls$ we have that $\lir{\fsl}\in\clus$ so $\lir{\fsl}$ is good, and hence there can be no more than $\icr$ such trials.
\item If $\lrep{i_t}{\lir{\fsl}}=1$ then since $\lir{\fsl}=j_{\fst}\in\clus$ and $j\in\clus$ we have $\lrep{i_t}{j}=\lrep{i_t}{\lir{\fsl}}=1$. Since $\lrel{i_t}{j}=0$ we then have $\lrel{i_t}{j}\neq\lrep{i_t}{j}$. So, since $j$ is good, there can be no more that $\icr$ such trials.
\end{itemize}
This has proven our claim that for all $j\in\clus$ the inequality $\nnul{\fsl}{j}\leq2\icr$ holds deterministically.

We now analyze the cardinality of $\cl{2}$. To do this consider some arbitrary $t\in\cl{2}$. Since $t\in\cls$ we have $j_t\in\clus$. Since $t>\fst$ and $j_t$ was not recommended to $i_t$ before trial $t$ we must have that either $\lrel{i_t}{\lir{\fsl}}=0$ or $\nnul{\fsl}{j_t}>2\icr$ (at some point). But $j_t\in\clus$ so, by above, $\nnul{\fsl}{j_t}\leq2\icr$ always holds so we must have $\lrel{i_t}{\lir{\fsl}}=0$. As $t\notin\fnb\cup\fnc$ we have $\lrep{i_t}{j_t}=1$ so since $j_t,\lir{\fsl}\in\clus$ we have $\lrep{i_t}{\lir{\fsl}}=1$. This implies $\lrel{i_t}{\lir{\fsl}}\neq\lrep{i_t}{\lir{\fsl}}$ and hence there can be at most $\icr$ possible values of $i_t$.

We have just shown that the cardinality of $\{i_{t}~|~t\in\cl{2}\}$ is at most $\icr$. Now note that on any $t\in\cl{2}$ we have $t\notin\fnl$, so $i_t$ is added to $\exc$ and hence cannot be equal to $i_{t'}$ for any future trial $t'\in\cl{2}$ with $t'>t$. Hence, for each $t\in\cl{2}$ we have that $i_t$ is unique, so the cardinality of $\cl{2}$ is equal to that of $\{i_{t}~|~t\in\cl{2}\}$, which is at most $\icr$. Since each $t\in\cl{2}$ is in $\fnd$ but not $\fnl$ it contributes $\bo{M+N}$ to the regret, so that $\cl{2}$ contributes 
$$
\bo{\icr (M+N)}
$$ 
to \algnn's the regret bound.

Suppose we have a trial $t\in\cl{2}\cup\cl{3}\setminus\{\fst\}$. If $\coin=0$ on trial $t$ then $t\in\cl{2}$, while if $\coin=1$ we have $t\in\cl{3}$. Since the probability that $\coin=1$ is $1/\icr$ and 
$|\cl{2}|\leq\icr$ we have 
$$
|\cl{3}|=\bo{1+|\cl{2}|/\icr}=\bo{1}
$$ 
in expectation. Since each trial $t\in\cl{3}$ contributes $\bo{M+\icr N}$ to the regret bound, this allows us to conclude that in expectation $\cl{3}$ contributes 
$$
\bo{M+\icr N}
$$ 
to \algnn's regret bound.

We now argue that we can exclude the contributions of $\fnb$ (and hence also $\cl{4}$) to the regret bound. To this effect, suppose we have some $t\in\fnd$ with $\rem{t}>2\ucp{i_t}$. Note that $j_t$ is drawn uniformly at random from the items not yet recommended to $i_t$ so far, and $\lrel{i_t}{j_t}=1$. This implies that each of the $\rem{t}$ items $j$ that user $i_t$ likes and have not been recommended to $i_t$ so far have a $1/\rem{t}$ probability of being $j_t$. Since at most $\ucp{i_t}$ of these items $j$ satisfy $\lrep{i_t}{j_t}=0$ we have that there is at most a 
$$
\ucp{i_t}/\rem{t}<1/2
$$
probability that $\lrep{i_t}{j_t}=0$ (so that $t\in\fnb$). Hence, $|\fnb|$ affects the regret bound by a constant factor only, so that we can exclude the contribution of $\cl{4}$.


Finally, suppose we have some $t\in\cl{5}$. Note that $t\in\fnc$ which implies that $i_t$ is bad. This means that $t\notin\cls$ which leads to a contradiction. The set $\cl{5}$ is therefore empty, hence it does not contribute to the regret.

Hence, the set $\cls$ contributes $\bo{\icr(M+N)}$ to the regret. Since the union of $\cls$ over all clusters $\clus$ is equal to the set of all $t\in\fnd$ such that $i_t$ and $j_t$ are both good, we have shown that the total expected regret is bounded by $\bo{D\icr(M+N)}$ plus the contribution (to the regret) of all $t\in\fnd$ in which either $i_t$ or $j_t$ is bad. Hence, we now analyze the contribution to the regret of the latter kinds of rounds.


Consider some bad user $i$ or bad item $j$, and let $\fbd$ be the set of trials $t\in\fnd$ with $i_t=i$ or $j_t=j$, respectively. Note that given $t\in\fbd$ with $t\notin\fnl$, in trial $t$ it must happen that both $i_t$ is added to $\exc$ and $j_t$ is added to $\exi$. This means that there can be no further trials in $\fbd$, hence there is at most one trial in $\fbd\setminus\fnl$. This also implies that whenever we encounter a round $t\in\fbd$, there is a $1-1/\icr$ probability that there are no further trials in $\fbd$ and a $1/\icr$ probability that $t\in\fnl$. This implies that the expected number of trials in $\fbd\cap\fnl$ is bounded from above by 
$$
\sum_{a\in\nat}1/\icr^a = \frac{\psi}{\psi-1}-1\leq 2/\icr~,
$$ 
where the inequality uses the condition $\icr\geq 2$. Since trials in $\fbd\setminus\fnl$ contribute $\bo{M+N}$ to the regret, and trials in $\fbd\cap\fnl$ contribute $\bo{M+\icr N}$, this shows that in expectation $\fbd$ contributes overall 
$$
\bo{M+N}
$$ 
to the regret.

Since there are $m$ bad users and $n$ bad items, the above shows that the set of trials $t\in\fnd$ such that $i_t$ or $j_t$ is bad contributes $\bo{(m+n)(M+N)}$ to the regret. Hence, 
$$
\expt{\reg}=\bo{(D\icr+m+n)(M+N)}~,
$$ 
as claimed.

We now single out the analysis changes for \ue. First note that we can, without loss of generality, assume there are no bad items. This is because for any bad item $j$ we can modify $\lmap$ so that its $j$-th column is equal to that of $\lmat$ noting that $\lmap$ becomes $(D+n)$-item clustered and there are now no bad items.

Now observe that, since no items are ever added to $\exi$, the condition in Line \ref{ntp-1} of Algorithm \ref{a:orcastar} is never true, so our regret is now 
$$
\expt{\reg} = \bo{(M+\icr N)|\fnl|+N|\fnd|}
$$
so trials in $\fnd\setminus\fnl$ now only contribute $\bo{N}$ instead of $\bo{M+N}$. Since there are no bad items (so $n=0$ and we can ignore in the analysis the fact that items are never added to $\exi$) this change leads to a regret bound of the form 
$$
\expt{\reg}=\bo{(D+m/\icr)(M+N\icr)}~,
$$ 
as claimed.
\end{proof}

\subsection{Doubling trick}\label{sa:doubling_trick}
We briefly detail the doubling trick needed to get rid of parameter $\psi$. 

For each value of $\icr$ in $\{2^a~|~a\in\nat\,,\, a\leq \log_2(M)+1\}$ take an instance of \algnn\ with that parameter value. On any trial we predict with and update only one instance $a$. We stay with instance $a$ until a mistake is made. Once a mistake is made we set $a\la a+1$ modulo $\lfloor\log_2(M)+1\rfloor$. This method allows us to achieve a regret bound that is only an $\bo{\ln(M)}$ factor off the regret bound of \algnn\ (Theorem \ref{t:orcastar}) with $\icr$ therein replaced by the optimal $\icr$ in hindsight.

\section{Further Empirical Results}\label{sa:further}
Table \ref{tab:addlabel} contains the area under the learning curve for all algorithms we tested. Table \ref{tab:addlabel2} contains running times.

\begin{table}[h]
\caption{Area under the curve (multiplied by 100) for all the methods we tested on the three versions of the MovieLens dataset. Standard errors over 30 repetitions are shown. Each column is tagged by the number $N$ of randomly selected items, along with the resulting (average) number $M$ of users. In bold is the best performance on each dataset, which turned out to be \orcapop*'s in all experiments we ran.}
  \label{tab:addlabel}%
  \centering
    \begin{tabular}{llll}
    \toprule
     & $N=50$  & $N=100$  & $N=200$ \\
    Method & $M\approx 3376$ &$M\approx 4488$ &$M\approx 5290$\\
    \midrule
    \random    & 50.08$\pm$0.07 & 50.01$\pm$0.04 & 50.03$\pm$0.03 \\
    \pop       & 60.11$\pm$0.96 & 59.49$\pm$0.88 & 58.69$\pm$0.55 \\
    \algn-IC       & 77.04$\pm$0.56 & 77.61$\pm$0.46 & 77.82$\pm$0.30 \\
    \algnn     & 74.58$\pm$0.67 & 76.87$\pm$0.39 & 78.06$\pm$0.26 \\
    \orcapop*  & {\bf 79.34$\pm$0.56} & {\bf 80.02$\pm$0.42} & {\bf 79.77$\pm$0.30} \\
    \wrmf-4    & 59.71$\pm$0.27 & 66.94$\pm$0.26 & 73.64$\pm$0.22 \\
    \wrmf-8    & 57.02$\pm$0.24 & 64.76$\pm$0.23 & 72.99$\pm$0.20 \\
    \wrmf-16   & 53.10$\pm$0.18 & 61.11$\pm$0.23 & 70.52$\pm$0.20 \\
    \wrmf-32   & 48.73$\pm$0.38 & 55.11$\pm$0.26 & 65.98$\pm$0.19 \\
    \midrule
    \end{tabular}%
\end{table}%

\begin{table}[t]
\caption{Average execution time per round (in milliseconds) for all methods on three versions of the MovieLens dataset. Standard errors over 30 repetitions are shown. Each column is tagged by the number $N$ of randomly selected items, along with the resulting (average) number $M$ of users. \wrmf\ methods turn out to be 50 to 100 times slower.}
  \label{tab:addlabel2}%
  \centering
    \begin{tabular}{llll}
    \toprule
      & $N=50$  & $N=100$  & $N=200$ \\
    Method & $M\approx 3376$ &$M\approx 4488$ &$M\approx 5290$\\
    \midrule
    \random & 0.0038 $\pm$ 0.0007 & 0.0070 $\pm$ 0.0005 & 0.0102 $\pm$ 0.0005 \\
    \pop & 0.0072 $\pm$ 0.0016 & 0.0166 $\pm$ 0.0016 & 0.0295 $\pm$ 0.0019 \\
    \algn-IC & 0.0043 $\pm$ 0.0009 & 0.0083 $\pm$ 0.0008 & 0.0124 $\pm$ 0.0007 \\
    \algnn & 0.0037 $\pm$ 0.0008 & 0.0070 $\pm$ 0.0006 & 0.0093 $\pm$ 0.0007 \\
    \orcapop* & 0.0074 $\pm$ 0.0014 & 0.0177 $\pm$ 0.0014 & 0.0310 $\pm$ 0.0018 \\
    \wrmf-4 & 1.0552 $\pm$ 0.0938 & 1.0339 $\pm$ 0.0271 & 0.6993 $\pm$ 0.0150 \\
    \wrmf-8 & 1.0880 $\pm$ 0.0746 & 0.9921 $\pm$ 0.1143 & 0.6934 $\pm$ 0.0153 \\
    \wrmf-16 & 1.0449 $\pm$ 0.1052 & 1.0757 $\pm$ 0.0450 & 0.7061 $\pm$ 0.0175 \\
    \wrmf-32 & 1.0050 $\pm$ 0.0690 & 1.0559 $\pm$ 0.1020 & 0.7555 $\pm$ 0.0169 \\
    \midrule
    \end{tabular}%
\end{table}%

\end{document}